\documentclass[]{bytedance_seed}

\usepackage{amsmath,amsfonts,bm}

\def\1{\bm{1}}

\def\rd{{\textnormal{d}}}

\DeclareMathAlphabet{\mathsfit}{\encodingdefault}{\sfdefault}{m}{sl}
\SetMathAlphabet{\mathsfit}{bold}{\encodingdefault}{\sfdefault}{bx}{n}

\DeclareMathOperator{\Tr}{Tr}

\newcommand{\Attn}{\textnormal{\textrm{Attn}}}
\newcommand{\FFN}{\textnormal{\textrm{FFN}}}
\newcommand{\RMSNorm}{\textnormal{\textrm{RMSNorm}}}
\newcommand{\Norm}{\textnormal{\textrm{Norm}}}
\newcommand{\trinorm}[1]{{\left\vert\kern-0.25ex\left\vert\kern-0.25ex\left\vert #1 
   \right\vert\kern-0.25ex\right\vert\kern-0.25ex\right\vert}}

\newcommand{\ba}{\boldsymbol{a}}

\newcommand{\bg}{\boldsymbol{g}}

\newcommand{\bp}{\boldsymbol{p}}
\newcommand{\bq}{\boldsymbol{q}}
\newcommand{\br}{\boldsymbol{r}}
\newcommand{\bs}{\boldsymbol{s}}

\newcommand{\bu}{\boldsymbol{u}}
\newcommand{\bv}{\boldsymbol{v}}
\newcommand{\bw}{\boldsymbol{w}}
\newcommand{\bx}{\boldsymbol{x}}

\newcommand{\bz}{\boldsymbol{z}}
\newcommand{\bA}{\boldsymbol{A}}
\newcommand{\bB}{\boldsymbol{B}}

\newcommand{\bD}{\boldsymbol{D}}

\newcommand{\bF}{\boldsymbol{F}}

\newcommand{\bH}{\boldsymbol{H}}
\newcommand{\bI}{\boldsymbol{I}}
\newcommand{\bJ}{\boldsymbol{J}}
\newcommand{\bK}{\boldsymbol{K}}

\newcommand{\bM}{\boldsymbol{M}}

\newcommand{\bP}{\boldsymbol{P}}
\newcommand{\bQ}{\boldsymbol{Q}}
\newcommand{\bR}{\boldsymbol{R}}

\newcommand{\bU}{\boldsymbol{U}}
\newcommand{\bV}{\boldsymbol{V}}
\newcommand{\bW}{\boldsymbol{W}}
\newcommand{\bX}{\boldsymbol{X}}

\newcommand{\balpha}{\bm{\alpha}}

\newcommand{\bgamma}{\bm{\gamma}}
\newcommand{\bepsilon}{\bm{\epsilon}}

\newcommand{\bmu}{\bm{\mu}}

\newcommand{\btheta}{\boldsymbol{\theta}}

\newcommand{\bxi}{\bm{\xi}}

\newcommand{\bSigma}{\boldsymbol{\Sigma}}

\newcommand{\cL}{\mathcal{L}}

\newcommand{\cN}{\mathcal{N}}
\newcommand{\cO}{\mathcal{O}}
\newcommand{\cP}{\mathcal{P}}
\newcommand{\cQ}{\mathcal{Q}}

\newcommand{\bbE}{\mathbb{E}}

\newcommand{\bbR}{\mathbb{R}}

\newcommand{\bzero}{\mathbf{0}}
\newcommand{\bone}{\mathbf{1}}
\newcommand{\pll}{\kern 0.56em/\kern -0.8em /\kern 0.56em}

\usepackage[toc,page,header]{appendix}

\usepackage[utf8]{inputenc} % allow utf-8 input
\usepackage[T1]{fontenc}    % use 8-bit T1 fonts
\usepackage{titletoc}       % wmz
\usepackage{booktabs}       % professional-quality tables
\usepackage{amsfonts}       % blackboard math symbols
\usepackage{nicefrac}       % compact symbols for 1/2, etc.
\usepackage{microtype}      % microtypography
\usepackage{xcolor}         % colors

\usepackage{graphicx}
\usepackage{etoc}
\usepackage{minitoc}

\usepackage{amsmath}
\usepackage{amssymb}
\usepackage{mathtools}
\usepackage{amsthm}
\usepackage{multirow}
\theoremstyle{plain}
\newtheorem{theorem}{Theorem}[section]

\theoremstyle{definition}

\newtheorem{setting}[theorem]{Setting}

\usepackage{wrapfig}

\title{ Negligible in Size, Significant in Effect: \\ On Scale Vectors in Large Language Models }

\author[2,*,\S,\dagger]{Mingze Wang}
\author[1,2,\S]{Shuchen Zhu}
\author[1]{Yuxin Fang}
\author[1,2]{\\Binghui Li}
\author[1]{Kai Shen}
\author[1, \dagger]{Shu Zhong}

\affiliation[1]{ByteDance Seed}
\affiliation[2]{Peking University}

\contribution[*]{Work done at ByteDance Seed}
\contribution[\S]{Equal contribution}
\contribution[\dagger]{Corresponding authors}
\abstract{
Normalization layers in modern large language models (LLMs) consist of a deterministic normalization operation and a learnable scale vector. While the normalization operation has been extensively studied, the scale vector remains poorly understood despite its ubiquitous use.
    In this work, we present a systematic study of scale vectors in LLMs from the perspectives of expressivity, optimization, and architectural structure.
    \textit{First}, we show empirically that although scale vectors constitute only a negligible fraction of model parameters, removing them substantially degrades LLM pre-training. Our theory further shows that, in Pre-Norm architectures, scale vectors do not increase expressivity; instead, they improve optimization through a self-amplifying preconditioning effect on subsequent linear mappings.
    \textit{Second}, we investigate the role of weight decay for scale vectors. By distinguishing Input-Norm and Output-Norm layers, we theoretically show that weight decay is beneficial for the former but harmful for the latter, due to their distinct roles in optimization and expressivity.
    \textit{Third}, motivated by this understanding, we propose three lightweight and complementary improvements to scale vectors: branch-specific heterogeneity, improved placement around linear mappings, and magnitude-direction reparameterization. Both theory and experiments show that each improvement yields consistent gains.
    \textit{Finally}, we combine these improvements into a unified scale-vector strategy and evaluate it through extensive LLM pre-training experiments on dense and mixture-of-experts models ranging from 0.12B to 2B parameters, across multiple optimizers and learning rate schedules, under industrial-scale token budgets. The unified strategy consistently achieves lower terminal loss than well-tuned baselines and exhibits more favorable scaling behavior, while adding negligible parameter and computational overhead. 
}

\date{\today}
\correspondence{\text{Mingze Wang at \email{mingzewang.math@gmail.com}, Shu Zhong at \email{zhongshu@bytedance.com}}}

\begin{document}
\maketitle

%不需要目录就注释掉 注意目录不要和第一页放在一块 要有\newpage
%\newpage
%\tableofcontents
%\newpage

\section{Introduction}

Normalization layers are a fundamental component of modern deep learning and are crucial for the stable and efficient training of large language models (LLMs)~\citep{vaswani2017attention}. In modern LLM architectures, normalization is typically implemented by RMSNorm~\citep{zhang2019root}, whose core structure consists of a deterministic normalization operation followed by a \textit{learnable scale vector} $\bgamma$. This simple normalize-then-scale structure has remained largely unchanged across modern LLMs.

The normalization operation itself has been extensively studied across deep learning.
Compared with BatchNorm (BN) in computer vision~\citep{ioffe2015batch} and LayerNorm (LN) in sequence modeling~\citep{ba2016layer}, RMSNorm~\citep{zhang2019root} modifies the normalization operation to improve simplicity and training stability in LLMs. 
However, its accompanying scale vector has largely retained its original form and received substantially less attention.

While the normalization operation is widely understood to stabilize training, the role of scale vectors remains unclear. 
Since scale vectors constitute only a negligible fraction of the total parameters, they are often treated as minor architectural details. However, do they truly have a negligible effect on LLM training? Answering this question is challenging, as it requires analyzing scale vectors jointly with other Transformer components from both expressivity and optimization perspectives.

In this work, we systematically investigate scale vectors in LLMs. We provide theoretical understanding of their mechanisms and propose scalable variants that improve LLM pre-training. \textbf{Our contributions} are summarized as follows:

\begin{itemize}
    \item \textbf{Necessity of scale vectors.} (Section~\ref{subsection: understanding necessity}) 
    We empirically show that scale vectors substantially affect LLM pre-training despite their negligible parameter count. Our theoretical analysis further reveals that: 
    (i) counterintuitively, scale vectors \textit{do not increase expressivity} in Pre-Norm architectures; (ii) their benefit arises from a \textit{self-amplifying preconditioning} effect on the subsequent linear mappings, which accelerates the training dynamics.
    
    \item \textbf{Weight decay of scale vectors.} (Section~\ref{subsection: understanding weight decay}) 
    We study an unresolved practical question: whether weight decay (wd) should be applied to scale vectors. We classify normalization layers into \textit{Input-Norm} and \textit{Output-Norm} according to whether their scale vectors are immediately followed by linear maps. 
    Our theoretical analysis shows that these two types have distinct mechanisms: 
    (i) for Input-Norm scale vectors, including those in Pre-Norm layers, wd is \textit{beneficial} as it induces balanced dynamics and controls Hessian sharpness, thereby accelerating and stabilizing training; 
    (ii) for Output-Norm scale vectors, wd is \textit{harmful} because these scale vectors determine expressivity, which wd can undesirably restrict. LLM experiments validate these theoretical insights.

    \item \textbf{Improving scale vectors.} (Section~\ref{section: improving}) 
    The above understanding motivates three complementary improvements to scale vectors, each with \textit{minimal} computational and parameter overhead.
    (i) \textbf{Heterogeneity.} In Transformer blocks such as self-attention, one shared Pre-Norm feeds multiple projections, including query, key, and value; however, these branches exhibit distinct training dynamics. We therefore introduce branch-specific scale vectors to provide \textit{tailored} self-amplifying preconditioners for different branches.
    (ii) \textbf{Placement.} In Pre-Norm architectures, scale vectors are consistently applied on the input side of the subsequent linear maps, inducing only row-wise self-amplifying preconditioning. We propose new placements that provably provide \textit{both row-wise and column-wise} preconditioning, further accelerate training.
    (iii) \textbf{Reparameterization.} We propose \textit{magnitude-direction reparameterizations} of scale vectors. Our theoretical analysis shows that these reparameterizations induce more anisotropic preconditioner and further accelerate training.
    Finally, we provide a \textit{unified preconditioning view} of these scale-vector designs.

    \item \textbf{Extensive experiments.} (Section~\ref{section: experiments}) 
    (i) We conduct systematic pre-training experiments on 0.12B Llama models to evaluate \textit{each strategy} and its variants, including weight decay, heterogeneity, placement, and reparameterization (Sections~\ref{subsection: understanding weight decay},~\ref{subsection: heterogeneity},~\ref{subsection: placement},~\ref{subsection: reparameterization}). Each strategy yields clear gains.
    (ii) We then evaluate a \textit{unified scale-vector strategy} that combines these improvements.
    We conduct extensive LLM pre-training experiments on both dense and mixture-of-experts (MoE) models, with sizes ranging from 0.12B to 2B parameters, 
    % trained on the high-quality FineWeb-Edu dataset 
    trained on high-quality pre-training corpus
    under industrial-scale token budgets. 
    We further evaluate different optimizers, including AdamW and Muon, and different learning rate schedules. 
    Across all settings, our unified strategy \textit{consistently achieves lower terminal loss} and \textit{exhibits more favorable scaling behavior} than well-tuned baselines, demonstrating its potential for improved scalability to larger models.
\end{itemize}

\section{Understanding Scale Vectors}
\label{section: understanding}

\textbf{Notations.} 
We denote the Hadamard product by $\odot$. For $\bx\in\bbR^d$, let $\Norm(\bx)=\sqrt{d}\bx/\|\bx\|$. For vectors, $\left<\cdot,\cdot\right>$ and $\|\cdot\|$ denote the Euclidean inner product and norm, respectively. For matrices, $\lambda_{\min}(\cdot)$, $\lambda_{\max}(\cdot)$, $\Tr(\cdot)$, and $\|\cdot\|_{\rm F}$
denote the smallest eigenvalue, largest eigenvalue, trace, and Frobenius norm, respectively.
For $\bv\in\bbR^d$, $\operatorname{diag}(\bv)\in\bbR^{d\times d}$ denotes the diagonal matrix with diagonal $\bv$.
We use $\cO(\cdot)$ to hide problem-independent constants. $\cN(\bmu,\bSigma)$ denotes the Gaussian distribution with mean $\bmu$ and covariance $\bSigma$. 
For a time-dependent quantity $h(t)$, write $\dot{h}:=\rd h/\rd t$.

We focus on RMSNorm, which is widely used in modern LLMs:
\begin{equation}
\RMSNorm(\bx;\bgamma)=\bgamma\odot\Norm(\bx),\quad\bx\in\bbR^d.
\end{equation}
Since BN and LN also contain scale vectors, similar mechanisms may apply beyond RMSNorm.

\subsection{Necessity of Scale Vectors}
\label{subsection: understanding necessity}

\textbf{Negligible in model size.}
Most Transformer parameters are matrix weights in the feedforward network (FFN) and self-attention (Attn) blocks, each of size $\cO(D^2)$, where $D$ is the hidden dimension.
In contrast, each scale vector has only $D$ parameters. For example, in the Llama-1B model considered in Section~\ref{section: experiments}, the model has 1,028,065,024 parameters in total, whereas all scale vectors together contain only 80,640 parameters, accounting for merely $\mathbf{7.84\times 10^{-5}}$ of the model size.

\textbf{No additional expressivity in Pre-Norm architectures.}
In Pre-Norm architectures such as Llama~\citep{touvron2023llama} (Figure~\ref{fig: llama, gemma}), each RMSNorm layer is immediately followed by a linear transformation. For example, the FFN block $\bX+\FFN(\RMSNorm(\bX;\bgamma))$ takes the form
\begin{equation}\label{eq: standard FFN}
\bX+ \bW_{\mathrm{down}}\big(\sigma(\bW_{\mathrm{gate}}(\bgamma\odot\Norm(\bX)))\odot \bW_{\mathrm{up}}(\bgamma\odot\Norm(\bX))\big).
\end{equation}
The same observation applies to Attn blocks and the final output layer. 
Since $\bgamma$ appears before linear maps, it can be absorbed into them: for any $\bgamma$ and linear map $\bW_2$, choosing $\bW_1=\bW_2\operatorname{diag}(\bgamma)$ yields
$\bW_1\Norm(\bx)\equiv \bW_2(\bgamma\odot\Norm(\bx))$.
Thus, from the perspective of expressivity, scale vectors in Pre-Norm architectures are redundant. This raises a natural question:
\begin{center}
\textit{Although negligible in size and expressively redundant, are scale vectors negligible in training?}
\end{center}

\begin{wrapfigure}{r}{0.61\textwidth}
    % \vspace{-.5cm}
    \includegraphics[width=0.33\textwidth]{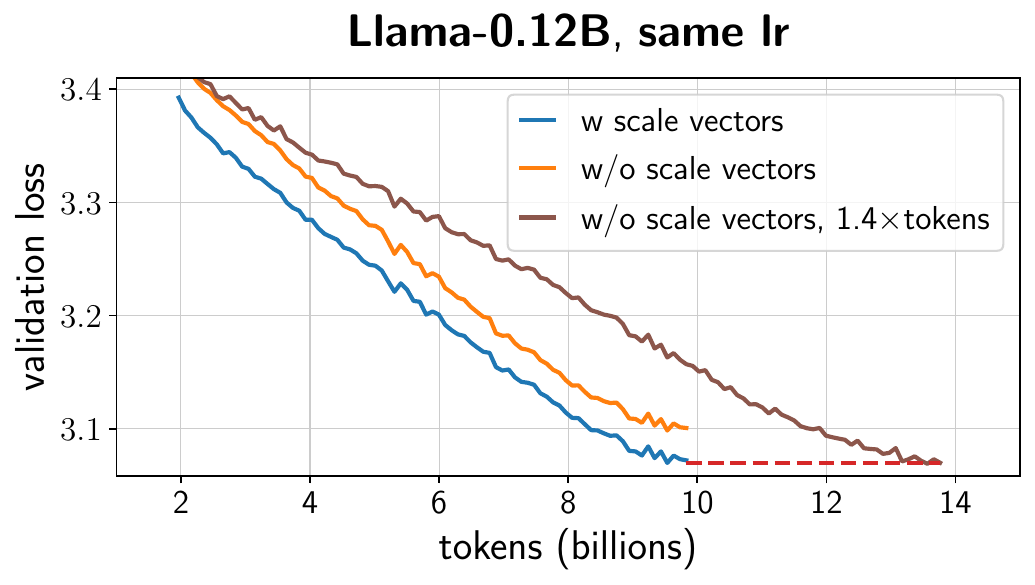}
    \hspace{-.15cm}
    \includegraphics[width=0.28\textwidth]{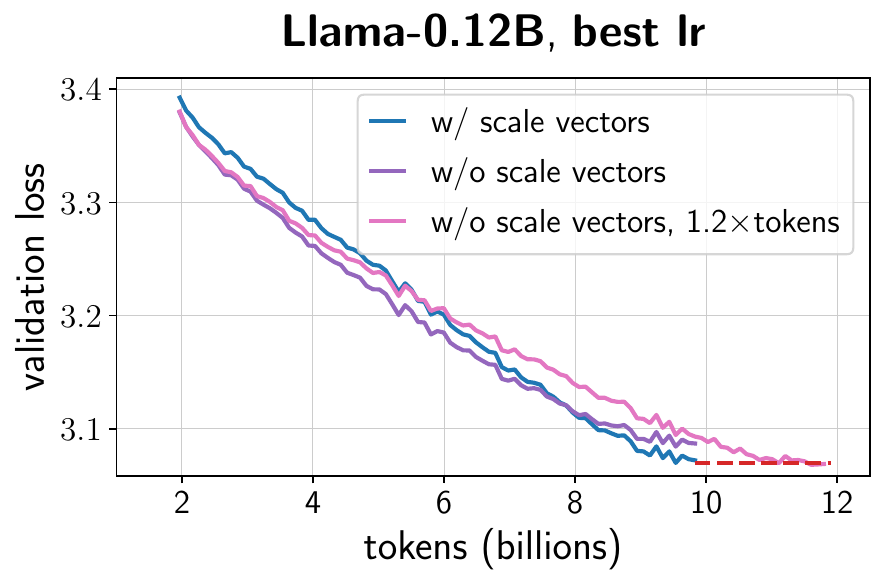}
    % \vspace{-.6cm}
    \caption{Ablation of scale vectors in 0.12B Llama pre-training. (Left): same peak lr; (Right): separately tuned peak lr.}
    \label{fig: w/o w/ scale vectors}
    % \vspace{-.3cm}
\end{wrapfigure}
\textbf{Indispensable in practice.} 
To answer this question, we conduct a simple ablation study by removing scale vectors. We train 0.12B Llama models;
% on FineWeb-Edu~\citep{lozhkovfineweb}; 
experimental details are provided in Appendix~\ref{appendix: subsection: understanding necessity}. 
As shown in Figure~\ref{fig: w/o w/ scale vectors}, under the same peak learning rate (lr), the model with scale vectors consistently outperforms the model without them throughout training, reducing the terminal loss by approximately $0.028$ and yielding a $1.4\times$ token-efficiency gain. Even after retuning the peak lr for the model without scale vectors, its terminal loss remains approximately $0.015$ higher.

\textbf{Theoretical study of training dynamics.}
These findings suggest that understanding scale vectors requires analyzing their effect on \textit{optimization dynamics}. 
We therefore study scale vectors in Pre-Norm architectures through the following simplified but representative setting~\citep{saxe2013exact}.

\begin{setting}\label{setting: w/ w/o scale vectors}
\text{Let $\bx\sim\cN(\bzero,\bI_d)$ and $f_\star(\bx)=\bW_\star\Norm(\bx)$ with some nonzero $\bW_\star\in\bbR^{c\times d}$.}
\end{setting}

We compare two student models under the squared loss.
The model with scale vectors is
$f(\bx;\bW_f,\bgamma)=\bW_f(\bgamma\odot\Norm(\bx))$ with
$\cL_f(\bW_f,\bgamma)=\frac12\bbE_{\bx}
\|f(\bx;\bW_f,\bgamma)-f_\star(\bx)\|^2$; the model without scale vectors is $g(\bx;\bW_g)=\bW_g\Norm(\bx)$ with
$\cL_g(\bW_g)=\frac12\bbE_{\bx}\|g(\bx;\bW_g)-f_\star(\bx)\|^2$.
Both objectives are optimized by gradient flow (GF): $(\dot{\bW}_f(t),\dot{\bgamma}(t))=-\nabla\cL_f(\bW_f(t),\bgamma(t))$ and $\dot{\bW}_g(t)=-\nabla\cL_g(\bW_g(t))$.
Following common practice, we initialize the scale vector as $\bgamma(0)=\bone$, while the remaining parameters have small magnitude. For clarity, we take $\bW_f(0)=\bW_g(0)=\bzero$.
Under this initialization, both models have the same
initial loss. For brevity, write $\cL_f(t):=\cL_f(\bW_f(t),\bgamma(t)),\cL_g(t):=\cL_g(\bW_g(t))$. 
We obtain the following result.

\begin{theorem}[Optimization advantage induced by scale vectors]
\label{thm: optimization advantage w/ scale vector} Consider Setting~\ref{setting: w/ w/o scale vectors}. Then, for every $t>0$, $\cL_f(t)<\cL_g(t)$.
\end{theorem}

% \vspace{-.1cm}

\textbf{Theoretical insight.} 
The proof of Theorem~\ref{thm: optimization advantage w/ scale vector} is deferred to Appendix~\ref{appendix: proof of understanding necessity}.
The key mechanism is that scale vectors transform the original gradient flow into a ``\textit{self-amplifying preconditioned}'' flow.
Specifically, define the effective parameter of $f$ as $\bA_f(t):=\bW_f(t)\operatorname{diag}(\bgamma(t))$, whose $j$-th column is $\ba_{f,j}(t):=\gamma_j(t)\bw_{f,j}(t)\in\mathbb R^c$.
Its dynamics satisfy $\dot{\ba}_{f,j}(t)=(\gamma_j(t)^2\bI_c+\bw_{f,j}(t)\bw_{f,j}(t)^\top)(\bw_{\star,j}-\ba_{f,j}(t))$.
By contrast, the model without scale vectors follows $\dot{\bw}_{g,j}(t)=\bw_{\star,j}-\bw_{g,j}(t)$.
Thus, the scale vector induces the state-dependent preconditioner $\bP_{f,j}(t):=\gamma_j(t)^2\bI_c+\bw_{f,j}(t)\bw_{f,j}(t)^\top$.
Under the practical initialization above, gradient flow preserves the invariant $\gamma_j(t)^2-\|\bw_{f,j}(t)\|^2=1$ and hence $\lambda_{\min}(\bP_{f,j}(t))\geq\gamma_j(t)^2\geq1$. 
Moreover, for every nonzero teacher column $\bw_{\star,j}\neq\bzero$, the inequality is strict for all $t>0$, yielding strictly faster loss descent.

% \vspace{-.1cm}

\begin{center}
\textbf{Takeaway.} \textit{Although scale vectors are negligible in parameter count and redundant in expressivity within Pre-Norm architectures, they accelerate training through self-amplifying preconditioning.}
\end{center}

\subsection{Weight Decay for Scale Vectors}
\label{subsection: understanding weight decay}

\textbf{An unresolved weight-decay choice.}
We next study the training dynamics of scale vectors, with a particular focus on weight decay (wd). WD is a standard component of modern LLM training and is typically applied to matrix parameters in Transformers. 
For scale vectors, however, the appropriate choice remains unclear. Existing LLM implementations, including OLMo, nanoGPT, and Qwen~\citep{olmo20242,Karpathy2022,yang2024qwen2technicalreport,yang2025qwen3}, differ in whether they apply wd to scale vectors. This motivates the question:
\begin{center}
% \vspace{-.2cm}
    \textit{Should weight decay be applied to scale vectors?}
% \vspace{-.2cm}
\end{center}

\begin{wrapfigure}{r}{0.44\textwidth}
    \centering
    \includegraphics[width=0.4\textwidth]{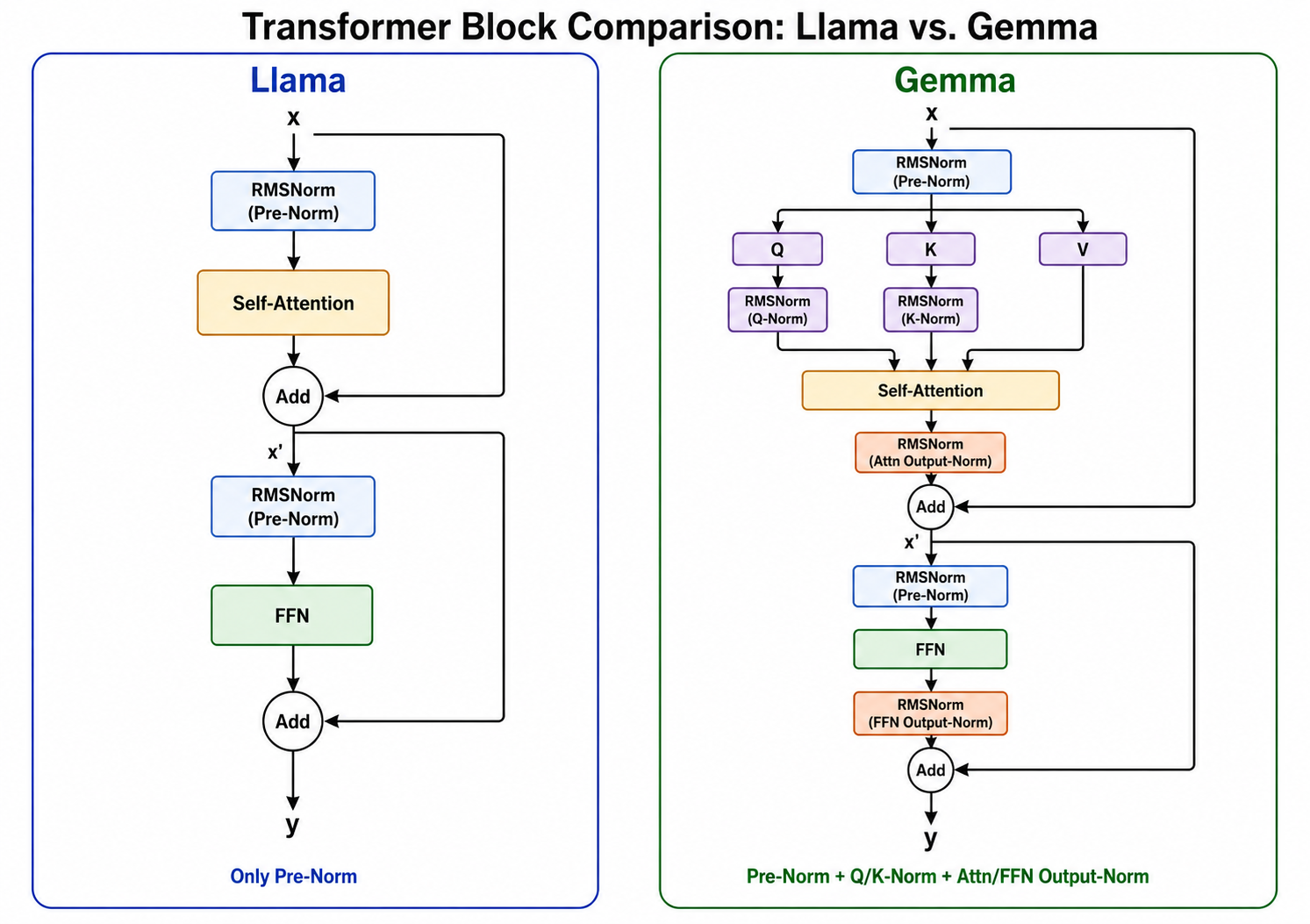}
    \caption{Comparison between Llama and Gemma architectures. 
    Llama contains only standard Pre-Norm layers, whereas Gemma additionally includes Q/K-norm and Attn/FFN Output-Norm layers. 
    % \textbf{Input-Norm v.s. Output-Norm.} According to our classification in Section~\ref{subsection: understanding weight decay}, Pre-Norm layers are Input-Norm layers, whereas Q/K-Norm and Attn/FFN Output-Norm layers are Output-Norm layers.
    }
    \label{fig: llama, gemma}
\end{wrapfigure}
\textbf{Distinct types of scale vectors.}
We primarily focus on scale vectors in Pre-Norm architectures. Since Q/K-Norm~\citep{dehghani2023scaling} and Output-Norm~\citep{team2024gemma} are increasingly used in recent LLMs, such as Qwen3 and Gemma3~\citep{gemmateam2025gemma3technicalreport}, we also include them in our study. 
Specifically, we consider Gemma3, which contains standard Pre-Norm, Q/K-Norm, and Attn/FFN Output-Norm layers.
Importantly, RMSNorm layers at different locations play distinct structural roles. We therefore classify the RMSNorm layers in Gemma3 into two types, as illustrated in Figure~\ref{fig: llama, gemma}:

\begin{itemize}
    % \vspace{-.2cm}
    \item \textbf{Output-Norm layers}: the RMSNorm layer not immediately followed by a linear transformation. This type includes Q/K-Norm and Attn/FFN Output-Norm. 
    % For example, the outer RMSNorm in $\bX+\RMSNorm(\FFN(\RMSNorm(\bX;\bgamma_{\rm in}));\bgamma_{\rm out})$ is and Output-Norm layer.
    % % \vspace{-.1cm}
    \item \textbf{Input-Norm layers}: the RMSNorm layer immediately followed by a linear transformation. This type includes standard Pre-norm layers. 
    % \vspace{-.2cm}
\end{itemize}

\textbf{Theoretical insight.} 
The scale vectors in these two types of RMSNorm layers play fundamentally different roles in expressivity and optimization, and should therefore be treated differently
under wd:

\begin{itemize}
    % \vspace{-.2cm}
    \item \textbf{Output-Norm scale vectors.} 
    These scale vectors directly parameterize the output of the corresponding submodule and therefore affect its expressivity. Applying wd to them shrinks the output, which can undesirably \textit{restrict expressivity} and weaken the submodule relative to the residual stream. Thus, weight decay should be avoided for Output-Norm scale vectors.

    \item \textbf{Input-Norm scale vectors.} 
    In contrast, these scale vectors add no expressivity (Section~\ref{subsection: understanding necessity}); their role is primarily optimization-related. 
    Applying weight decay to them keeps the \textit{parameterization balanced} and \textit{controls Hessian sharpness}, thereby leading to potentially faster and more stable training. Theorem~\ref{thm: wd, pre-norm} provides a detailed analysis.
    
    % \vspace{-.2cm}
\end{itemize}

We now theoretically study whether wd should be applied to Input-Norm scale vectors. Unlike Section~\ref{subsection: understanding necessity}, which analyzes GF, we consider the more realistic setting of stochastic gradient descent (SGD) through its continuous-time SDE approximation. For clarity, we present the scalar-output model $f(\bx;\bw,\bgamma)=\langle\bw,\bgamma\odot\Norm(\bx)\rangle$ with target $f^\star(\bx)=\langle\bw^\star,\Norm(\bx)\rangle$ under Setting~\ref{setting: w/ w/o scale vectors}.

\begin{theorem}[Benefits of wd for Input-Norm scale vectors; informal]\label{thm: wd, pre-norm}
Consider training the above model by continuous-time SGD. Suppose wd \(\lambda>0\) is applied to \(\bw\), and compare two choices for the Input-Norm scale vector \(\bgamma\): no wd (\(\mu=0\)) versus wd (\(\mu>0\)). 
\begin{itemize}
    % \vspace{-.15cm}
    \item If \(\mu>0\), then \(\mathbb E\|\bgamma_t\|^2<\infty\) remains uniformly bounded over time. Consequently, key Hessian sharpness metrics, including \(\lambda_{\max}(\nabla^2\cL)\), \(\Tr(\nabla^2\cL)\), and \(\|\nabla^2\cL\|_{\mathrm F}\), also \textbf{remain bounded}.
    % \vspace{-.15cm}
    \item If \(\mu=0\), then \(\mathbb E\|\bgamma_t\|^2\) becomes unbounded along the trajectory. As a result, the same Hessian sharpness metrics \textbf{diverge} along a sequence of times.
    % \vspace{-.2cm}
\end{itemize}
\end{theorem}

\begin{wrapfigure}{r}{0.485\textwidth}
     % \vspace{-.6cm}
    % \includegraphics[width=0.25\textwidth]{figures/understanding/qwen_0.5B_10BT.pdf}
    \includegraphics[width=0.245\textwidth]{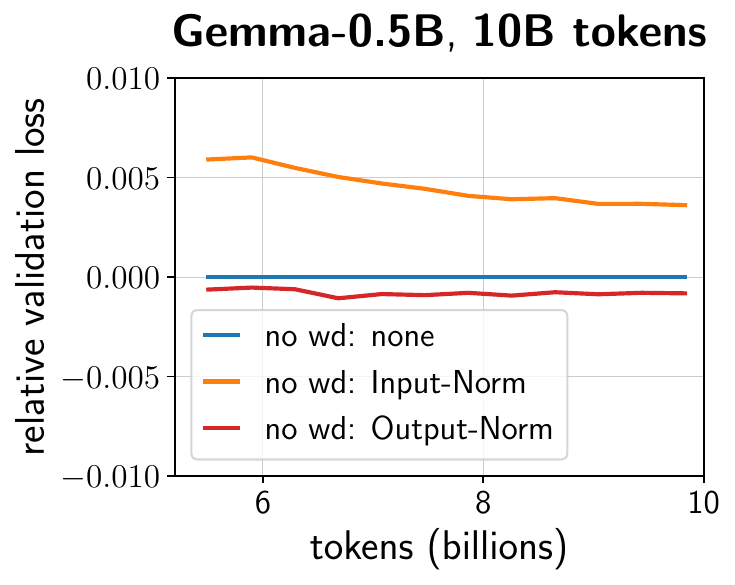}
    \hspace{-.25cm}
    \includegraphics[width=0.245\textwidth]{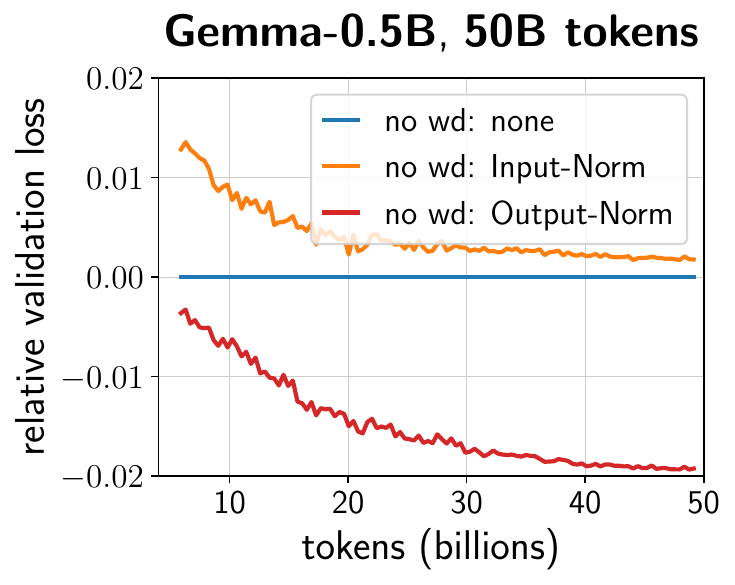}
    % \vspace{-.5cm}
    \caption{WD is beneficial for Input-Norm scale vectors, but harmful for Output-Norm scale vectors.}
    % \vspace{-.4cm}
    \label{fig: understanding: weight decay}
\end{wrapfigure}
\textbf{Key mechanism.} The formal statement and proof are deferred to Appendix~\ref{appendix:wd-pf}. Theorem~\ref{thm: wd, pre-norm} shows that, for Input-Norm scale vectors, wd keeps the parameterization pair \((\bw,\bgamma)\) balanced, prevents excessive growth of \(\bgamma\), and thereby controls Hessian sharpness during training.

\textbf{Why sharpness matters.}
In stochastic optimization, the loss descent of SGD \textit{depends directly} on Hessian sharpness metrics, such as \(\lambda_{\max}(\nabla^2\cL)\), \(\Tr(\nabla^2\cL)\), and \(\|\nabla^2\cL\|_{\mathrm F}\). 
Therefore, by preventing these quantities from diverging, wd on Input-Norm scale vectors leads to \textit{faster training} and \textit{potentially more stable} optimization, including the use of larger lr.

We outline the main mechanism below and refer to Appendix~\ref{appendix:wd-hessian} for details.
Let $\cL(\cdot)$ be a twice-differentiable loss function, and the SGD update can be written as
$\btheta_{t+1}=\btheta_t-\eta\left(\nabla\cL(\btheta_t)+\bxi_t\right)$,
where the stochastic noise satisfies $\bbE\big[\bxi_t\big]=\bzero$.
A second-order Taylor expansion yields the expected loss descent:
\begin{align*}
\bbE[\cL(\btheta_{t+1})] = \bbE\big[\underbrace{\cL(\btheta_t)-\eta\|\nabla \cL(\btheta_t)\|^2+\frac{\eta^2}{2}\nabla \cL(\btheta_t)^\top \nabla^2\cL(\btheta_t)\nabla\cL(\btheta_t)}_{\text{GD term}}\big]  + \frac{\eta^2}{2} \underbrace{\bbE\big[\bxi_t^\top \nabla^2\cL(\btheta_t)\bxi_t\big]}_{\text{noise term}} + o(\eta^2),
\end{align*}
The GD term is contributed by deterministic GD, which is governed by the Hessian spectral norm $\lambda_{\max}(\nabla^2\cL(\btheta_t))$, since $\nabla\cL(\btheta_t)^\top\nabla^2\cL(\btheta_t)\nabla\cL(\btheta_t)\leq\lambda_{\max}(\nabla^2\cL(\btheta_t))\|\nabla\cL(\btheta_t)\|^2$; 
The noise term is induced by gradient noise $\bxi_t$,
different assumptions on the covariance $\bbE[\bxi_t\bxi_t^\top]$ lead to different  contribution. However, under a broad class of SGD noise modeling -- including bounded variance, isotropic noise, Hessian-aligned noise, and others~\citep{hazan2016introduction,bottou2018optimization,feng2021inverse,welling2011bayesian,haochen2021shape,mori2022power,wang2023noise} -- this term is consistently controlled by Hessian sharpness metrics such as 
$\lambda_{\max}(\nabla^2\cL(\btheta_t))$, $\Tr(\nabla^2\cL(\btheta_t))$, and
$\|\nabla^2\cL(\btheta_t)\|_{\rm F}^2$.

\textbf{Experimental justification.}
To validate our theoretical insights, we independently control wd for Input-Norm
and Output-Norm scale vectors. 
We train 0.5B Gemma models
% on FineWeb-Edu 
for 10B/50B tokens using Muon optimizer~\citep{jordan2024muon}; experimental details are provided in Appendix~\ref{appendix: subsection: understanding weight decay}. 
As shown in Figure~\ref{fig: understanding: weight decay}, applying wd to Input-Norm scale vectors consistently improves performance; whereas removing wd from Output-Norm scale vectors performs better than applying it.
These results support the following principle.

% \vspace{-.1cm}
\begin{center}
\textbf{\! Individual weight decay} 
\!(\textbf{IWD}). \textit{Apply wd to Input-Norm scale vectors, but not to Output-Norm ones.}
\end{center}
% \vspace{-.2cm}

\section{Improving Scale Vectors}
\label{section: improving}

In this section, we propose three methods to further exploit scale vectors. For each method, we first present the motivation and theoretical insight. 
In Section~\ref{subsection: unified view}, we provide a unified view of these designs, and in Section~\ref{section: experiments}, we empirically verify that each design consistently improves LLM pre-training.

\subsection{Heterogeneity of Scale Vectors}
\label{subsection: heterogeneity}

\textbf{Shared scale vectors in Attn and FFN.}
We consider Pre-Norm architecture. In the Attn block, $\bX+\Attn(\RMSNorm(\bX;\bgamma))$, a single Pre-Norm layer $\RMSNorm(\bX;\bgamma)$ feeds all three projections:
\begin{equation}\label{eq: standard QKV}
    \bQ=\bW_Q(\bgamma\odot\Norm(\bX)),\quad \bK=\bW_K(\bgamma\odot\Norm(\bX)),\quad \bV=\bW_V(\bgamma\odot\Norm(\bX))
\end{equation}
Thus, the scale vector $\bgamma$ is shared across the query, key, and value branches. The FFN block follows the same pattern: a single Pre-Norm output feeds both the gate and up projections.

\textbf{Heterogeneous scale vectors} (\textbf{HG}).
Figure~\ref{fig: HG: qkv dynamics} shows that the query, key, and value matrices exhibit different growth and decay rates during training. 
As Section~\ref{subsection: understanding necessity} shows, in Pre-Norm scale vectors primarily accelerate optimization through self-amplifying preconditioning. 
From this perspective, different
branches should be equipped with separate scale vectors, allowing the induced preconditioners to adapt to branch-specific optimization dynamics.
This motivates replacing the shared scale vector in~\eqref{eq: standard QKV} with branch-specific scale vectors $\bgamma_Q$, $\bgamma_K$, and $\bgamma_V$:
\begin{equation}
\bQ=\bW_Q(\bgamma_{Q}\odot\Norm(\bX)),\ \bK=\bW_K(\bgamma_{K}\odot\Norm(\bX)),\ \bV=\bW_V(\bgamma_{V}\odot\Norm(\bX)).
\end{equation}

% % 进一步，可以每个 head 都不一样
% Head-wise
% \[
% \bQ^{(h)}=\bW_{Q^{(h)}}(\bgamma_Q^{(h)}\odot\Norm(\bX))\in\bbR^{D/H},\quad h\in[H]
% \]
% % limit 每个维度都用不同的 scale vector
% Dimension-wise
% \[
% \bQ^{i}={\bw_{Q}^{i}}^\top(\bgamma_Q^{i}\odot\Norm(\bX))\in\bbR,\quad h\in[H]
% \]
% Equal: Omit $\Norm(\bX)$,
% \[
% \bW=(\bw_1,\cdots,\bw_d)^\top\to
% (\bw_1\odot\bgamma,\cdots,\bw_d\odot\bgamma)^\top
% \to (\bw_1\odot\bgamma_1,\cdots,\bw_d\odot\bgamma_d)=\bW\odot\bGamma
% \]

\begin{wrapfigure}{r}{0.27\textwidth}
     % \vspace{-.65cm}
     \includegraphics[width=0.245\textwidth]{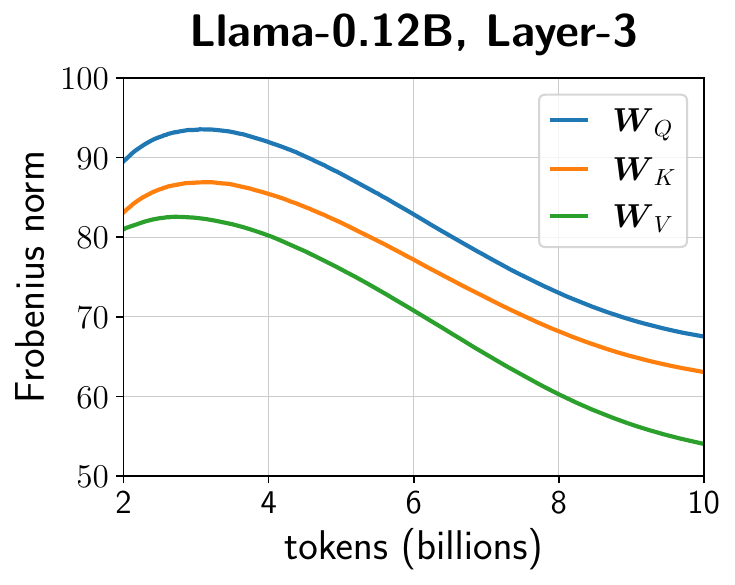}
    % \vspace{-.3cm}
    \caption{ Dynamics of query, key, and value projections.}
    % \vspace{-.2cm}
    \label{fig: HG: qkv dynamics}
\end{wrapfigure}
An analogous modification applies to the FFN block. Specifically, we introduce separate scale vectors $\bgamma_{\mathrm{gate}}$ and $\bgamma_{\mathrm{up}}$ into the standard FFN in~\eqref{eq: standard FFN}, yielding:
\begin{equation*}
    \bX+\bW_{\mathrm{down}}\big(\sigma(\bW_{\mathrm{gate}}(\bgamma_{\mathrm{gate}}\odot\Norm(\bX)))\odot \bW_{\mathrm{up}}(\bgamma_{\mathrm{up}}\odot\Norm(\bX))\big).
\end{equation*}
This heterogeneity does not increase expressivity; its role is purely optimization-related. Moreover, it introduces only $\cO(d)$ additional parameters, which is negligible relative to the overall model size.

\subsection{Placement of Scale Vectors}
\label{subsection: placement}

\textbf{Standard input-side placement.}
We consider Pre-Norm architectures. As shown in~\eqref{eq: standard FFN} and~\eqref{eq: standard QKV}, scale vectors in Attn or FFN blocks are placed before linear maps. 
The same holds for the RMSNorm layer before the output projection. Thus, Pre-Norm architectures consistently apply scale vectors on the input side of a linear map. 
This convention, however, is not obviously optimal for optimization.

\textbf{Alternative placement for faster optimization.}
Section~\ref{subsection: understanding necessity} shows that standard input-side scale vectors already accelerate optimization relative to removing them. A natural question is whether alternative placements can yield an even stronger optimization effects.
The key observation is that the placement of a scale vector determines which coordinates are directly modulated. 
Input-side placement rescales coordinates before they are mixed by $\bW$, but the resulting output coordinates may still exhibit heterogeneous scales or anisotropic optimization
geometry. 
Hence, the standard placement may not fully exploit the optimization benefits of scale vectors.
Motivated by this observation, we consider three increasingly flexible designs.

\begin{itemize}
\item \textbf{AP:} \textbf{after-placement.}
The simplest modification moves the scale vector from the input side to the output side of the linear map:
\begin{equation}
    \textbf{AP:}\qquad\bW(\bgamma^{\text{b}}\odot\Norm(\bx))\mapsto\bgamma^{\text{a}}\odot(\bW\Norm(\bx)).
\end{equation}
Compared with the standard design, AP directly modulates output coordinates after linear mixing, and may therefore better match output-side anisotropy.

\item \textbf{DP:} \textbf{dual-placement.}
AP moves the scale vector to the output side but removes input-side modulation. A more flexible alternative is to place scale vectors on both sides of the linear map:
\begin{equation}
\textbf{DP:}\qquad
\bW(\bgamma^{\text{b}}\odot\Norm(\bx))\mapsto\bgamma^\text{a}\odot(\bW(\bgamma^\text{b}\odot\Norm(\bx))).
\end{equation}
Since scale vectors are negligible in size, this two-sided design incurs only a tiny parameter increase while controlling both input-side and output-side coordinates.

\item \textbf{DNP:} \textbf{dual normalized placement.}
Although DP enables more flexible modulation around the linear map, the extra multiplicative interactions may reduce training stability. 
We therefore consider a normalized variant that inserts normalization between $\bW$ and the output-side scale vector:
\begin{equation}
\textbf{DNP:}\qquad\bW(\bgamma^{\text{b}}\odot\Norm(\bx))\mapsto\bgamma^\text{a}\odot\Norm(\bW(\bgamma^\text{b}\odot\Norm(\bx))).
\end{equation}
Compared with DP, DNP preserves two-sided modulation while explicitly normalizing the intermediate representation, which may stabilize optimization.
For attention projections, normalization is applied separately to each head;
for other projections, normalization is applied directly to the hidden states.
\end{itemize}

When applied to the query and key projections, \textbf{DNP} is equivalent to adding Q/K-Norm. 
For other matrices, it introduces additional normalization layers. Moreover, when using heterogeneous input-side scale vectors from Section~\ref{subsection: heterogeneity}, it is natural to also use heterogeneous output-side scale vectors.

The following theorem demonstrates that \textbf{DP} can further accelerate the optimization dynamics of scale vectors. Its proof is deferred to Appendix~\ref{appendix: proof of improving: placement}.

\begin{theorem}[Instantaneous acceleration of \textbf{DP}]
\label{thm: optimization advantage placement}
Consider the setting of Theorem~\ref{thm: optimization advantage w/ scale vector}. Let $\phi(\bx;\bW_\phi,\bgamma^{\rm a},\bgamma^{\rm b})=\bgamma^{\rm a}\odot(\bW_\phi(\bgamma^{\rm b}\odot\Norm(\bx)))$ be the  DP student model, and let $\cL_\phi$ be its squared loss. 
We compare the standard model $f$ and the DP model $\phi$ under gradient flow.
Then,  at the same effective state, \textbf{DP} is instantaneously \textbf{at least as fast as} the standard model: for any $(\bW_f,\bgamma)$ and$(\bW_\phi,\bgamma^{\rm a},\bgamma^{\rm b})$ satisfying $\bgamma=\bgamma^{\rm b}$ and $\bW_f=\operatorname{diag}(\bgamma^{\rm a})\bW_\phi$, we have $ \frac{\rd}{\rd t}\cL_\phi\leq\frac{\rd}{\rd t}\cL_f\leq 0$.
\end{theorem}

\textbf{Key mechanism.}
Theorem~\ref{thm: optimization advantage placement} establishes the instantaneous acceleration of DP at the same effective state as the baseline. 
The key mechanism is that DP induces \textit{both row-wise and column-wise self-amplifying preconditioning}, leading to faster loss descent than the standard single-sided modulation.
A finer trajectory-level analysis, including strict acceleration in both early and global training phases, is deferred to Theorem~\ref{thm: full: optimization advantage placement} in Appendix~\ref{appendix: proof of improving: placement}.

\subsection{Reparameterization of Scale Vectors}
\label{subsection: reparameterization}

\textbf{Why parameterization matters.}
Another key question is how scale vectors should be parameterized. Even when different parameterizations represent the same functional class, they can induce substantially different optimization dynamics. 
This issue is relevant for scale vectors, which naturally contain two distinct components: a global magnitude and a directional unit vector. 
These components play different optimization roles. 
The magnitude controls the overall strength of modulation, whereas the direction determines how this modulation is distributed across channels. 
The standard parameterization entangles these two effects in a single vector, potentially leading to suboptimal conditioning. 
This motivates reparameterizing scale vectors to control magnitude and direction more explicitly.

\textbf{Magnitude-direction Reparameterization.} For nonzero scale vector $\bgamma\in\bbR^d$, we decompose it into its magnitude and direction, and optimize the two components separately. We consider two variants.

\begin{itemize}
\item \textbf{OR:} \textbf{reparameterization in original space.}
A natural approach is to separate magnitude and direction in the standard Euclidean geometry:
\begin{equation}
    \textbf{OR:}\qquad\bgamma \mapsto \beta\cdot\Norm(\balpha),
\end{equation}
where $\beta\in\bbR$ controls the global magnitude, and $\Norm(\balpha)\in\bbR^d$ determines the direction.

\item \textbf{ER:} \textbf{reparameterization in exponential space.}
While \textbf{OR} separates magnitude and direction in Euclidean coordinates, scale vectors act multiplicatively in the network. 
This suggests that a multiplicative parameterization may be more natural. We therefore consider
\begin{equation}
    \textbf{ER:}\qquad\bgamma \mapsto e^\beta\cdot e^{\balpha-\text{mean}(\balpha)}
\end{equation}
where $e^\beta\in\bbR$ controls the overall magnitude, and $e^{\balpha-\text{mean}(\balpha)}\in\bbR^d$ determines the direction.
\end{itemize}

Although ER produces only positive entries and therefore appears to remove sign information, it does not reduce expressivity in our setting: since scale vectors always appear together with linear maps, the missing signs can be absorbed into the corresponding linear weights.

The following theorem demonstrates that \textbf{OR} can further accelerate optimization dynamics induced by scale vectors. Its proof is deferred to Appendix~\ref{appendix: proof of improving: reparameterization}.

\begin{theorem}[Optimization advantage of \textbf{OR}]
\label{thm: optimization advantage reparametrization}
Consider the setting of Theorem~\ref{thm: optimization advantage w/ scale vector}. 
Let $\psi(\bx;\bW_\psi,\balpha,\beta)=\bW_\psi(\beta\Norm(\balpha)\odot\Norm(\bx))$ be the reparameterized student model, and let $\cL_\psi$ be its squared loss. 
We compare the standard model $f$ and the reparameterized model $\psi$ under gradient
flow, initialized as $\balpha(0)=\bone,\beta(0)=1,\bW_\psi(0)=\bzero$. 
Then, at the same effective state, \textbf{OR} is instantaneously \textbf{at least as fast as} the direct parameterization: for any $(\bW_f,\bgamma)$ and $(\bW_\psi,\balpha,\beta)$ satisfying $\bW_f=\bW_\psi$ and $\bgamma=\beta\Norm(\balpha)$, we have $ \frac{\rd}{\rd t}\cL_\psi\leq\frac{\rd}{\rd t}\cL_f\leq 0$.
\end{theorem}

\textbf{Key mechanism.}
Theorem~\ref{thm: optimization advantage reparametrization} establishes the instantaneous acceleration of \textbf{OR} at the same effective state as the baseline.
The key mechanism is that \textbf{OR} induces an \textit{anisotropic preconditioner} that treats the magnitude and direction components of the scale vector differently.
Under gradient flow, the effective dynamics of $\bgamma:=\beta\Norm(\balpha)$ in $\psi$ satisfy $\dot{\bgamma}= -\bP \nabla_{\bgamma} \cL_f$, where the implicit preconditioner is $\bP = d \cdot (\hat{\bgamma}\hat{\bgamma}^\top) + \rho \cdot (\bI - \hat{\bgamma}\hat{\bgamma}^\top)$, where $\hat{\bgamma} = \bgamma/\|\bgamma\|$ and $\rho = \beta^2 d / \|\balpha\|^2 \approx \cO(1)$. 
Unlike the standard parameterization ($\bP = \bI$), this structure amplifies motion along the magnitude direction by a factor of $d$, aggressively accelerating scale adjustments, while preserving a stable $\cO(1)$ descent scale for the orthogonal direction.
A finer trajectory-level analysis, including strict acceleration in early and global training phases, is deferred to Theorem~\ref{thm: full: optimization advantage reparametrization} in Appendix~\ref{appendix: proof of improving: reparameterization}
and further supported by numerical simulations.

\subsection{A Unified Preconditioning View of Scale-Vector Designs}
\label{subsection: unified view}

The preceding subsections introduce three different scale-vector designs: heterogeneity, placement, and reparameterization. 
We now show that they can be understood under a unified principle. 
For Pre-Norm scale vectors $\bgamma$, the scale vector is immediately followed by a matrix $\bW$ and therefore does not enlarge the expressivity of the corresponding linear transformation. 
Instead, it changes the \textit{ parameterization} of the effective matrix optimized during training. Consequently, these designs act as \textit{lightweight preconditioners}.

\textbf{Preconditioning via reparameterization.}
We first illustrate the general connection between reparameterization and preconditioning.
Let \(\bp\in\bbR^{m}\) denote an effective parameter that enters the loss \(\cL(\bp)\). 
If \(\bp\) is optimized directly, its gradient flow is:
\begin{equation}\label{eq: gf, standard}
    \frac{\rd\bp}{\rd t}=-\nabla_{\bp}\cL(\bp).
\end{equation}
Now consider that \(\bp\) is parameterized by another variable \(\bq\in\bbR^{n}\) with $n\ge m$ through a differentiable map
\[
    \bp = \Phi(\bq), \qquad \Phi:\bbR^{n}\to\bbR^{m}.
\]
Under this reparameterization, by the chain rule, 
$\nabla_{\bq}\cL(\Phi(\bq))=\frac{\partial\Phi(\bq)}{\partial\bq}^\top \nabla_{\bp}\cL(\bp)$.
Therefore, gradient flow on the parameter \(\bq\), \(\frac{\rd\bq}{\rd t}=-\nabla_{\bq}\cL(\Phi(\bq))\), induces the following dynamics on the effective parameter \(\bp=\Phi(\bq)\):
\begin{equation}\label{eq: gf, reparameterization}
    \frac{\rd\bp}{\rd t}=
    \frac{\partial\Phi(\bq)}{\partial\bq}\frac{\rd\bq}{\rd t}=-\frac{\partial\Phi(\bq)}{\partial\bq}\frac{\partial\Phi(\bq)}{\partial\bq}^\top
    \nabla_{\bp}\cL(\bp).
\end{equation}
Comparing~\eqref{eq: gf, reparameterization} with~\eqref{eq: gf, standard}, we see that reparameterization induces a state-dependent preconditioned gradient flow in the effective parameter space, with preconditioner
\begin{equation}\label{eq: preconditioner via reparameterization}
    \bP_{\Phi}(\bq)
    :=
    \frac{\partial\Phi(\bq)}{\partial\bq}\frac{\partial\Phi(\bq)}{\partial\bq}^\top.
\end{equation}
This preconditioner \(\bP_{\Phi}(\bq)\) is positive semi-definite and state-dependent.

\textbf{Reparameterization forms of scale-vector designs.}
We now illustrate the above preconditioning view for our scale-vector designs.
Consider a group of linear branches $\{\bW_c\}_c$ following a Pre-Norm layer, such as
\(c\in\{Q,K,V\}\) in an Attn block, \(c\in\{\mathrm{gate},\mathrm{up}\}\) in a FFN block, or a single linear branch such as the final output projection.
For a branch matrix \(\bW_c\in\bbR^{d_2\times d_1}\), a general class of lightweight reparameterizations of the effective matrix can be written as
\begin{equation}\label{eq: preconditioner unified}
    \bW_c
    \mapsto
    \operatorname{diag}(\bu_c)\bW_c\operatorname{diag}(\bv_c)
    =
    \bW_c\odot(\bu_c\bv_c^\top),
\end{equation}
where $\bu_c\in\bbR^{d_2},\bv_c\in\bbR^{d_1}$.
This form introduces a structured multiplicative scale field over the entries of \(\bW_c\in\bbR^{d_2\times d_1}\).
Notably, \(\bu_c\bv_c^\top\) at most rank one, so the design uses only \(\cO(d_1+d_2)\) additional parameters while modulating an \(\cO(d_1d_2)\)-dimensional matrix.
% For standard Transformer projections, the parameter overhead is \(\cO(d)\) due to $d_1\sim d_2$, negligible compared with the \(\cO(d^2)\) matrix parameters.

Table~\ref{tab: unified preconditioner} summarizes how our scale-vector strategies fit into this unified form, including the standard scale vector, \textbf{HG}, \textbf{AP}, \textbf{DP}, \textbf{OR}, and \textbf{ER}. 
The only exception is \textbf{DNP} ($\operatorname{diag}(\bu_c)\operatorname{Norm}(\bW_c\operatorname{diag}(\bv_c)\,\cdot)$), which is not a purely reparameterization because it inserts an intermediate normalization to further stabilize the hidden representation. 
Although each row in Table~\ref{tab: unified preconditioner} presents a single strategy, many strategies are complementary and can be combined to obtain richer scale-vector designs.

\begin{table}[!htp]
\centering
\caption{Scale-vector designs as special cases of the unified reparameterization 
\(\bW_c\mapsto\bW_c\odot(\bu_c\bv_c^\top)\).}
\label{tab: unified preconditioner}
\small
\begin{tabular}{lll}
\toprule
Design & Choice in the unified form~\eqref{eq: preconditioner unified} & Preconditioning interpretation \\
\midrule
Standard 
& \(\bu_c=\mathbf 1,\ \bv_c=\bgamma\) 
& Input-side channel preconditioning \\
\textbf{HG} 
& \(\bu_c=\mathbf 1,\ \bv_c=\bgamma_c\) 
% depends on branch \(r\) 
& Branch-specific preconditioning \\
\textbf{AP} 
& \(\bu_c=\bgamma^{\text{a}},\ \bv_c=\mathbf 1\) 
& Output-side channel preconditioning \\
\textbf{DP} 
& \(\bu_c=\bgamma^{\text{a}},\ \bv_c=\bgamma^{\text{b}}\) 
& Two-sided input/output preconditioning \\
% DNP 
% & \(\mathcal N_r^{\rm mid}=\operatorname{Norm}\) 
% & Stabilized two-sided preconditioning \\
\textbf{OR} 
& \(\bu_c=\mathbf 1,\ \bv_c=\beta\cdot\operatorname{Norm}(\balpha)\) 
& Magnitude-direction anisotropic metric \\
\textbf{ER} 
& \(\bu_c=\mathbf 1,\ \bv_c=e^\beta\cdot e^{\balpha-\operatorname{mean}(\balpha)}\) 
& Multiplicative/log-space scale metric \\
\bottomrule
\end{tabular}
\end{table}

\textbf{Comparison with adaptive optimizer preconditioning.}
Modern adaptive optimizers, such as RMSProp~\citep{Tieleman2012_rmsprop}, Adam~\citep{kingma2014adam}, KFAC~\citep{pmlr-v37-martens15}, Shampoo~\citep{pmlr-v80-gupta18a}, and Soap~\citep{vyas2025soap}, introduce an \textit{optimizer-induced preconditioner} $\bF_t^{-1}$ in the raw parameter space.
Recent variants further improve optimization by enhancing dynamics along flat directions of the preconditioned geometry~\citep{wang2024improving,wang2025sharpness,wang2025gradpower,zhu2026accelerating}. 
Under adaptive optimization, the gradient flow $\frac{\rd \bp}{\rd t}=-\nabla_{\bp}\cL(\bp)$ is modified to $\frac{\rd \bp}{\rd t}=-\bF_t^{-1}(\bp)\nabla_{\bp}\cL(\bp)$,
where $\bF_t$ is estimated from gradient statistics accumulated during training. 
For example, RMSProp approximately uses $\bF_t=\operatorname{diag}
(
\sqrt{\operatorname{EMA}_t[(\nabla_{\bp}\cL(\bp_t))^{\odot2}]}
)$.

In contrast, a reparameterization $\Phi$ induces a
\textit{reparameterization-induced preconditioner}
$\bP_{\Phi}$, as in~\eqref{eq: preconditioner via reparameterization}. 
This preconditioner depends on the
current model state, but not explicitly on historical gradients. Table~\ref{tab: comparison: reparameterization vs optimizer} summarizes the differences.
Thus, the two forms of preconditioning arise from \textit{fundamentally different mechanisms}: reparameterization reshapes the local optimization geometry through the current model state, whereas adaptive optimizers adjust updates according to accumulated gradient statistics.

\begin{table}[!htp]
    \centering
    \caption{Comparison between optimizer-induced and reparameterization-induced preconditioning.}
    \label{tab: comparison: reparameterization vs optimizer}
    \small
    \begin{tabular}{lcc}
    \toprule
    & Adaptive optimizer & Reparameterization $\Phi$ \\
    \midrule
    Preconditioner
    &
    $\bF_t^{-1}(\bp_t)$
    &
    $\bP_{\Phi}(\bq_t)
    =
    \frac{\partial\Phi(\bq_t)}{\partial\bq}
    \frac{\partial\Phi(\bq_t)}{\partial\bq}^{\top}$
    \\
    Determined by
    &
    gradient history
    $\{\nabla_{\bp}\cL(\bp_s)\}_{s\le t}$
    &
    current model state $\bq_t$
    \\
    \bottomrule
    \end{tabular}
\end{table}

\textit{Combining these two} mechanisms further alters the effective preconditioning geometry and may yield additional optimization gains. Intuitively, reparameterization provides a more suitable coordinate system, while adaptive optimizers exploit gradient statistics within that coordinate system. The experiments in Section~\ref{section: experiments} support this complementary effect.

\section{Experiments}
\label{section: experiments}

% \vspace{-.05cm}

% \subsection{Experimental Setup}

We evaluate our methods on LLM pre-training across architectures and model scales. The main configurations are summarized below; additional implementation details are provided in Appendix~\ref{appendix: section: experiments}.

\textbf{Models and Data.} 
We conduct experiments on two widely used LLM architectures: \textbf{dense} models (Llama~\citep{touvron2023llama}) and \textbf{MoE} models (LlamaMoE), with sizes \textbf{from 0.12B to 2B} parameters. All models are trained on  
a high-quality curated pre-training corpus.
% the high-quality FineWeb-Edu dataset~\citep{lozhkovfineweb}. 

\textbf{Token Budget.}
Unless otherwise specified, the training budget is approximately \textit{100} tokens per parameter for dense models and \textit{100} tokens per activated parameter for MoE models. 
This budget substantially exceeds the Chinchilla-optimal regime~\citep{hoffmann2022training} and is closer to industrial pre-training practice. 
For small-scale models, we use a total batch size of approximately 0.5M tokens; for large-scale models, we use a large total batch size of approximately 2M tokens.

\textbf{Optimization.}
In most experiments, we adopt AdamW as the default optimizer~\citep{kingma2014adam,loshchilov2017decoupled}. Following standard Llama pre-training practice~\citep{touvron2023llama}, we set $\beta_1=0.9$, $\beta_2=0.95$, weight decay $\lambda=0.1$, and gradient clipping threshold $1.0$. 
We use the \texttt{cos} schedule: linear warmup to a peak lr \texttt{lr\_max}, followed by cosine decay to the terminal lr \texttt{lr\_min}. For each setting, we tune \texttt{lr\_max} for the baseline model and use the same value for both the baseline and the corresponding improved model.

To validate our methods beyond this default setting, we additionally report results with the Muon optimizer~\citep{jordan2024muon} and the warmup-stable-decay (wsd) scheduler~\citep{hu2024minicpm}.

% \vspace{-.05cm}

\subsection{Validation of Each Strategy}
\label{subsection: validate proposed method}

% \vspace{-.05cm}

We first evaluate the strategies introduced in Sections~\ref{section: understanding} and~\ref{section: improving}: \textbf{HG} for scale-vector heterogeneity (Section~\ref{subsection: heterogeneity}); \textbf{AP}, \textbf{DP}, and \textbf{DNP} for scale-vector placement (Section~\ref{subsection: placement}); \textbf{OR} and \textbf{ER} for scale-vector reparameterization (Section~\ref{subsection: reparameterization}); and \textbf{IWD} for individual weight decay of scale vectors (Section~\ref{subsection: understanding weight decay}).

\begin{figure}[!htb]
    \centering
    % \vspace{-.2cm}
    \includegraphics[width=0.245\linewidth]{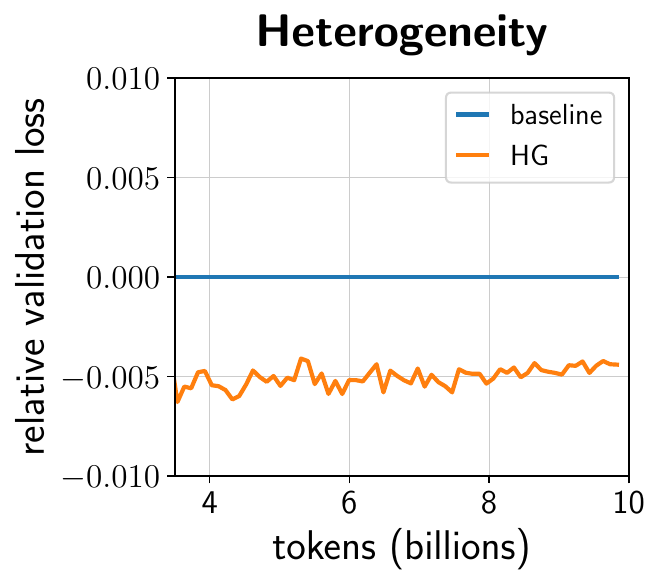}
    \includegraphics[width=0.245\linewidth]{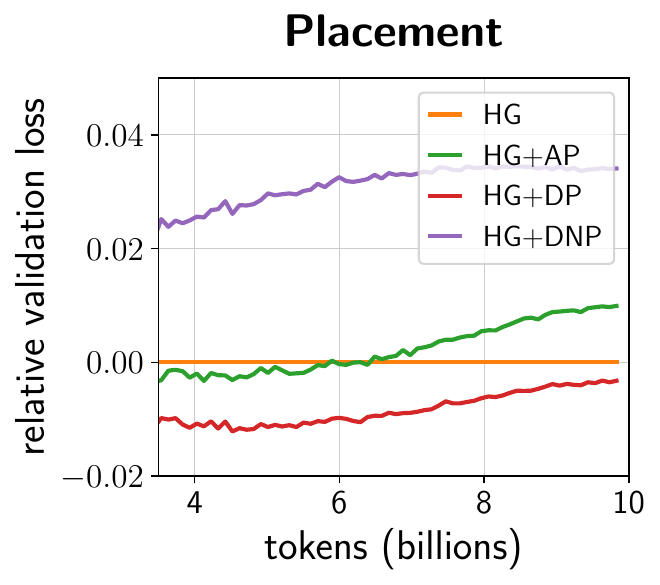}
    \includegraphics[width=0.245\linewidth]{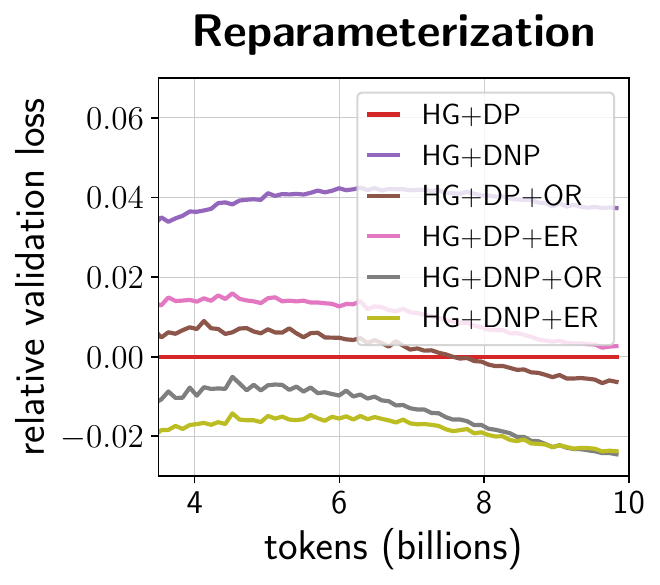}
    \includegraphics[width=0.246\linewidth]{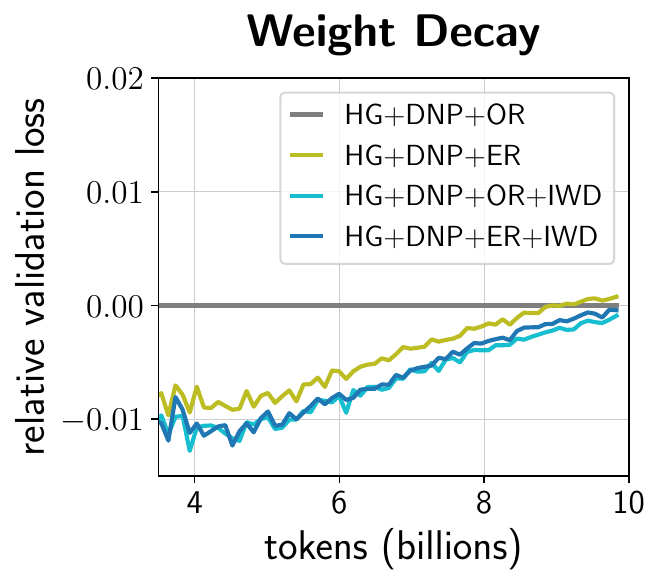}
    % \vspace{-.5cm}
    \caption{Step-by-step validation of scale-vector strategies on Llama-0.12B. Each strategy class yields gains.}
    % \vspace{-.2cm}
    \label{fig: validation of each strategy}
\end{figure}

\textbf{Step-by-step verification.}
All experiments in this subsection are conducted on the 0.12B Llama model, and the strategies are validated incrementally. The results are shown in Figure~\ref{fig: validation of each strategy}. 
(i) The \textit{Heterogeneity} panel shows that heterogeneous scale vectors outperform shared scale vectors.
(ii) The \textit{Placement} panel shows that dual placement (\textbf{DP}) consistently improves performance, whereas dual normalized placement (\textbf{DNP}) does not. After placement (\textbf{AP}) improves early training but is later overtaken.
(iii) The \textit{Reparameterization} panel shows that applying magnitude-direction reparameterization, either \textbf{OR} or \textbf{ER}, improves performance. Among these variants, \textbf{DNP+OR} achieves the largest gain, outperforming \textbf{DP+OR} and \textbf{DP+ER}, and slightly outperforming \textbf{DNP+ER}.
(iv) Although vanilla Llama contains only Pre-Norm layers, which are Input-Norm layers, applying \textbf{DNP} introduces both Input-Norm and Output-Norm scale vectors. 
This allows us to re-examine individual weight decay in this setting. As shown in the \textit{Weight Decay} panel, \textbf{OR+IWD} achieves the lowest terminal loss.
Overall, these results show that all four classes of scale-vector strategies provide gains, further supporting our theoretical understanding.

\subsection{Main Results of Unified Strategy}
\label{subsection: main results of unified strategy}

Building on the step-by-step validation in Section~\ref{subsection: validate proposed method}, we combine \textbf{HG} (Section~\ref{subsection: heterogeneity}), \textbf{DNP} (Section~\ref{subsection: placement}), \textbf{OR} (Section~\ref{subsection: reparameterization}), and \textbf{IWD} (Section~\ref{subsection: understanding weight decay}) into a \textit{unified scale-vector strategy}.

In this subsection, we systematically compare this strategy with tuned baselines across dense and MoE models ranging from 0.12B to 2B parameters. Across all settings, our method consistently improves pre-training
loss.

\begin{figure}[!htb]
    \centering
    \includegraphics[width=0.3\linewidth]{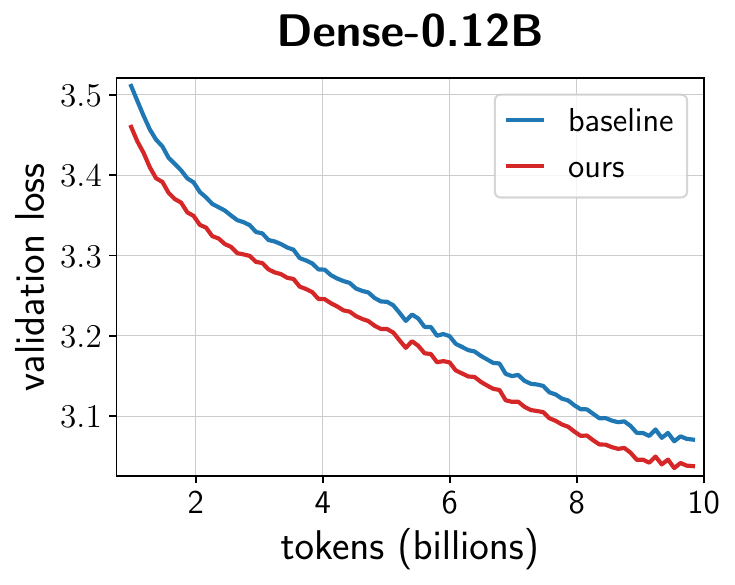}
    % \hspace{-.15cm}
    \includegraphics[width=0.3\linewidth]{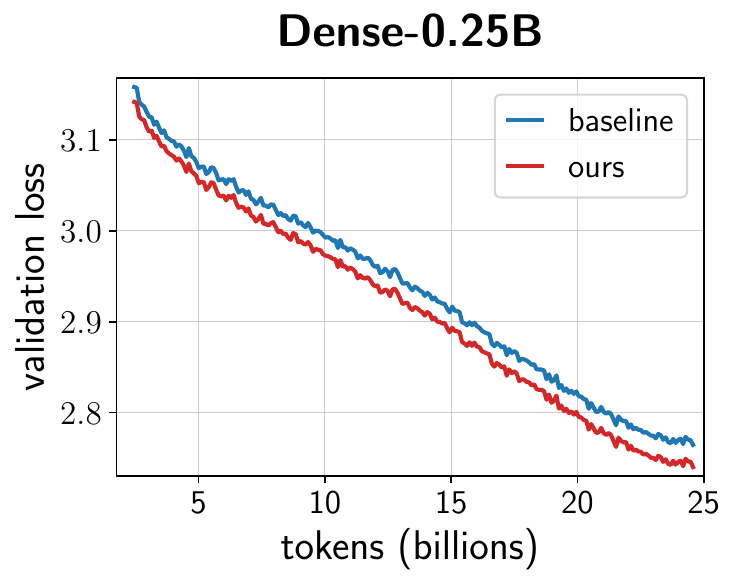}
    % \hspace{-.15cm}
    \includegraphics[width=0.3\linewidth]{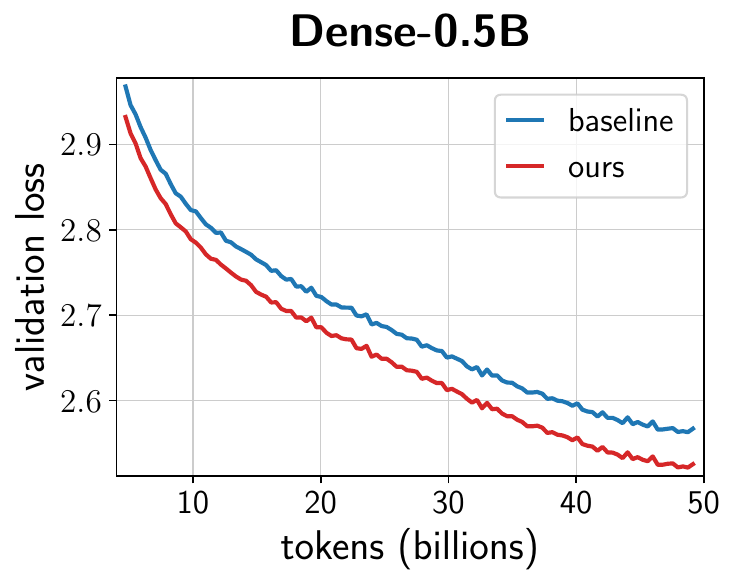}
    % \hspace{-.15cm}
    \includegraphics[width=0.305\linewidth]{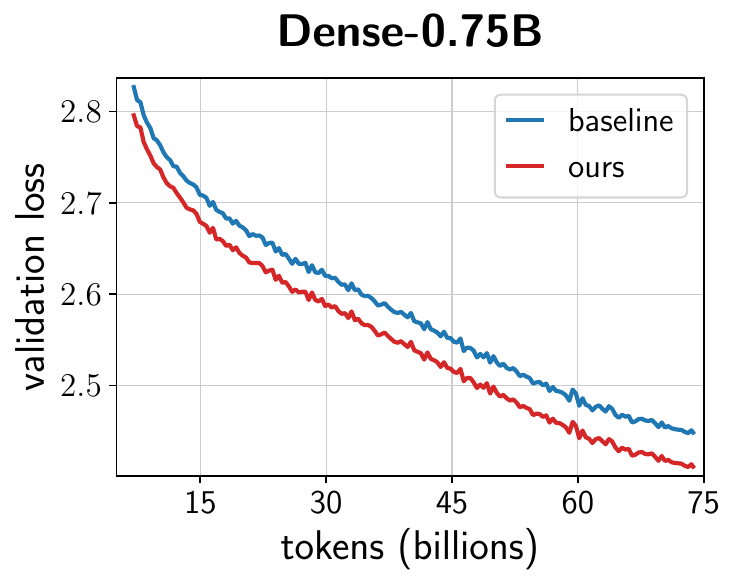}
    % \hspace{-.15cm}
    \includegraphics[width=0.3\linewidth]{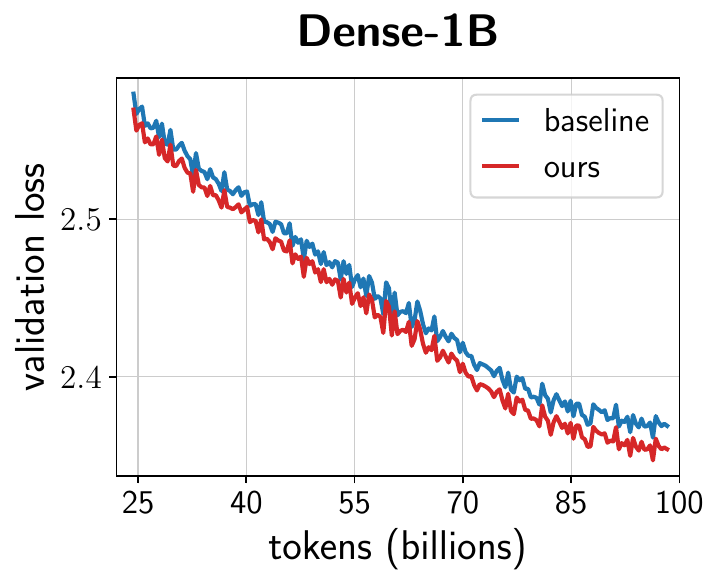}
    \includegraphics[width=0.238\linewidth]{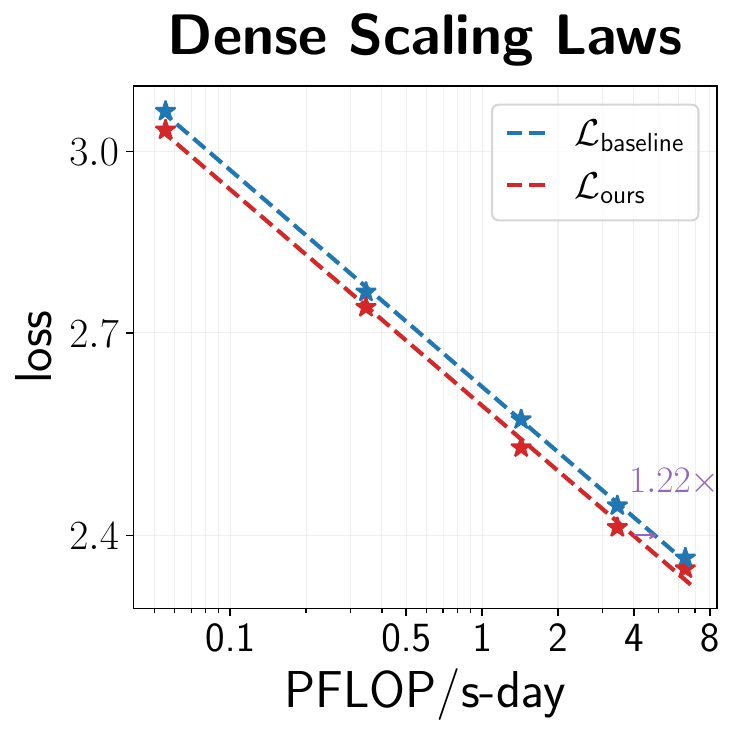}
    \hspace{.9cm}
    
    \caption{Comparison between our unified scale-vector strategy and the Llama baseline on dense models of different sizes. Our strategy consistently achieves lower terminal loss across model sizes and scales better than the baseline.}
\label{fig: dense: unified strategy}
\end{figure}

\begin{figure}[!htb]
    \centering
    % \vspace{-.2cm}
    \includegraphics[width=0.3\linewidth]{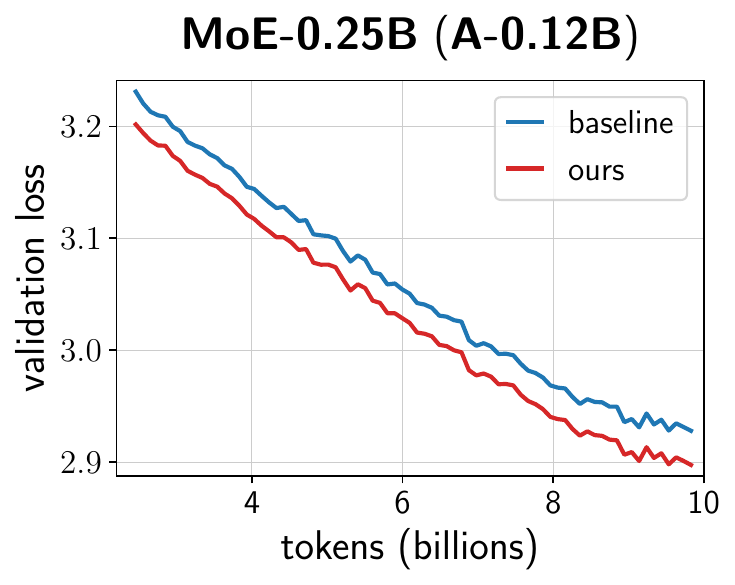}
    % \hspace{-.15cm}
    \includegraphics[width=0.294\linewidth]{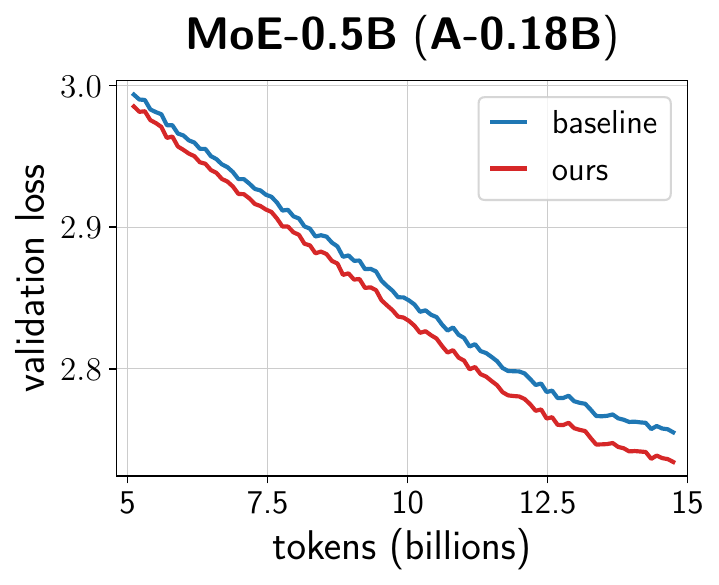}
    % \hspace{-.15cm}
    \includegraphics[width=0.3\linewidth]{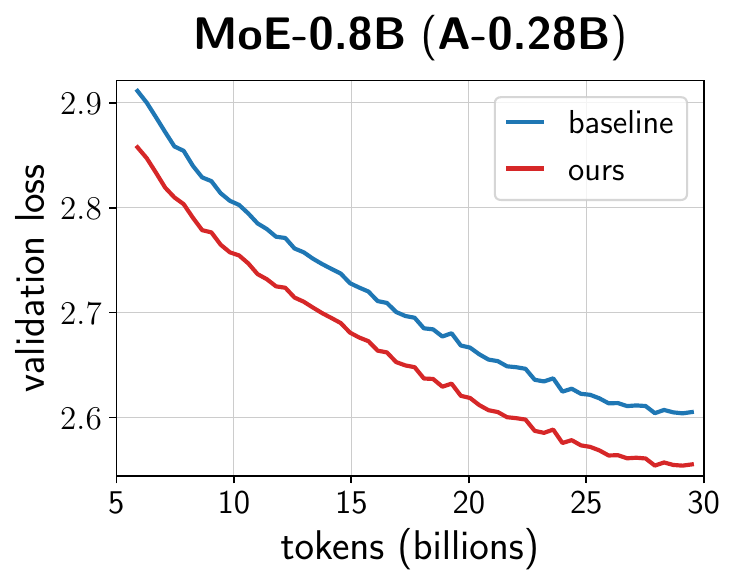}
    % \hspace{-.15cm}
    \includegraphics[width=0.3\linewidth]{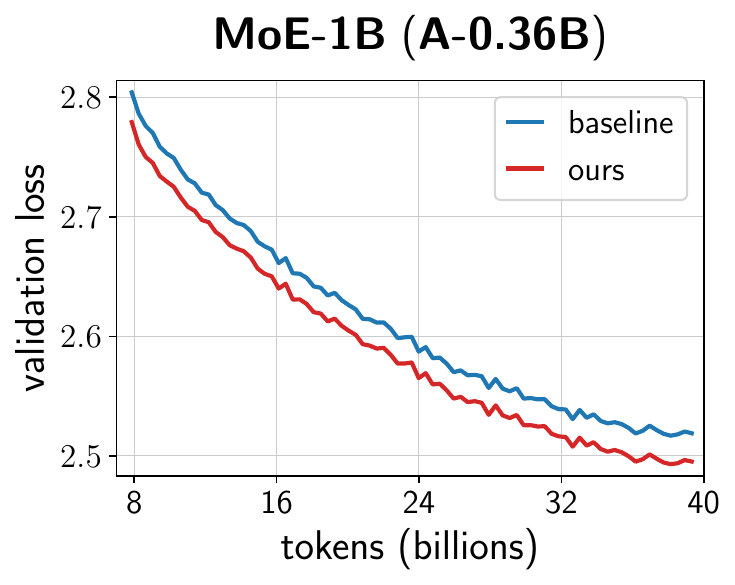}
    % \hspace{-.15cm}
    \includegraphics[width=0.3\linewidth]{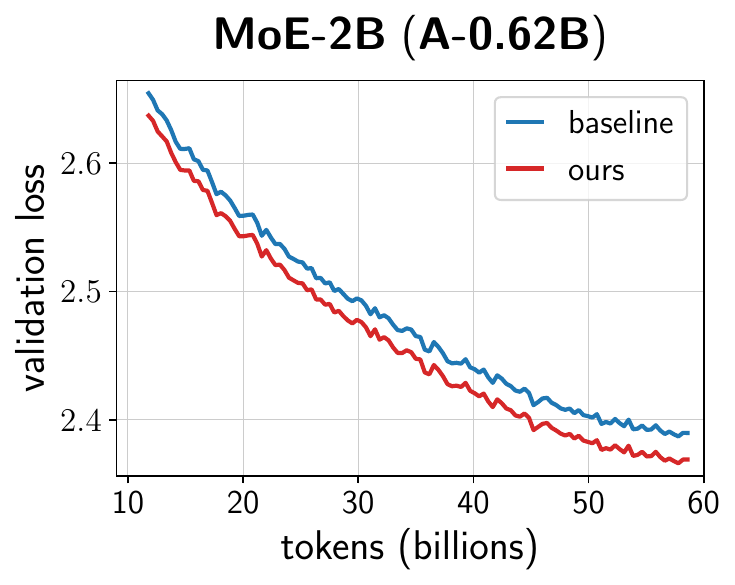}
    % \vspace{-.4cm}
    \includegraphics[width=0.233\linewidth]{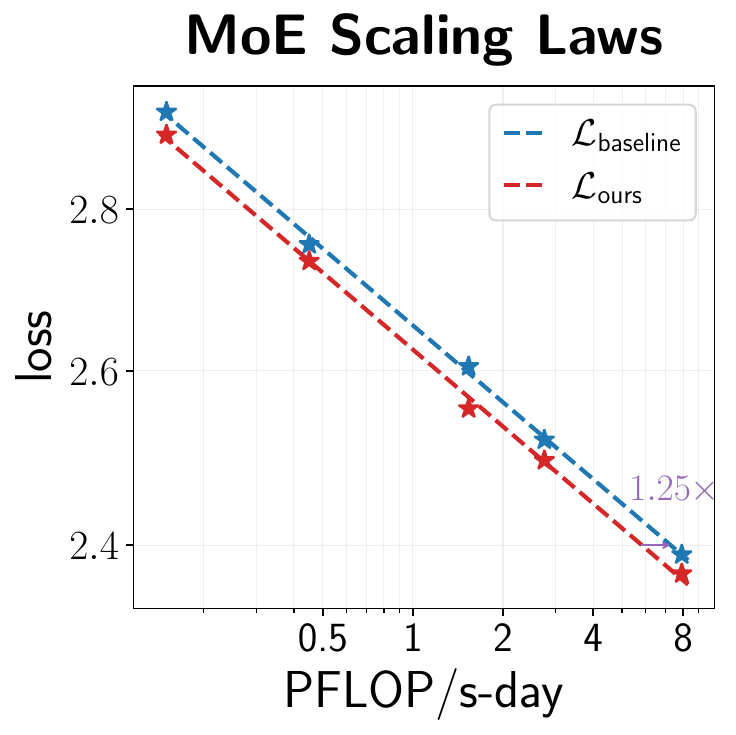}
    \hspace{.9cm}
    \caption{Comparison between our unified scale-vector strategy and the Llama-MoE baseline on MoE models of different sizes. Our strategy consistently achieves lower terminal loss across model sizes and scales more favorably than the baseline.}
    % \vspace{-.2cm}
\label{fig: moe: unified strategy}
\end{figure}

\textbf{Results for Dense Models.}
Figure~\ref{fig: dense: unified strategy} compares our method with tuned Llama baselines on 
% FineWeb-Edu for
dense Llama models with 0.12B, 0.25B, 0.5B, 0.75B, and 1B parameters; 0.75B results are deferred to Appendix~\ref{appendix: subsection: main results of unified strategy}.
Our method \textbf{consistently achieves lower terminal loss} than the tuned Llama baseline and maintains its advantage throughout training, with the gap widening over time.
The fitted scaling laws also suggest that these gains may persist at larger scales.

\textbf{Results for MoE Models.}
% \textbf{Main findings.} 
Figure~\ref{fig: moe: unified strategy} compares our unified scale-vector strategy with tuned Llama-MoE baselines 
% on FineWeb-Edu 
across MoE models with 0.25B (A0.11B), 0.5B (A0.18B), 0.8B (A0.28B), 1B (A0.38B), and 2B (A0.62B) parameters, where ``A'' denotes the number of activated parameters; 0.8B results are deferred to Appendix~\ref{appendix: subsection: main results of unified strategy}.
Across all settings, our method \textbf{consistently achieves more than 0.02 lower terminal loss} than the well-tuned MoE baseline.
Moreover, the advantage is maintained throughout training, with the performance gap widening over time.
Furthermore, our method also exhibits a \textbf{more favorable scaling law} than the baseline. In particular, the fitted loss-descent slope of our strategy is \textit{slightly steeper} than that of the baseline.

\subsection{Compatibility with Muon and wsd}
\label{subsection: muon and wsd}

We further examine whether our strategy is compatible with other training configurations, including the Muon optimizer and the wsd lr schedule. Experimental details are provided in Appendix~\ref{appendix: subsection: muon and wsd}.

\begin{figure}[!htb]
    \centering
    % \vspace{-.2cm}
    \includegraphics[width=0.3\linewidth]{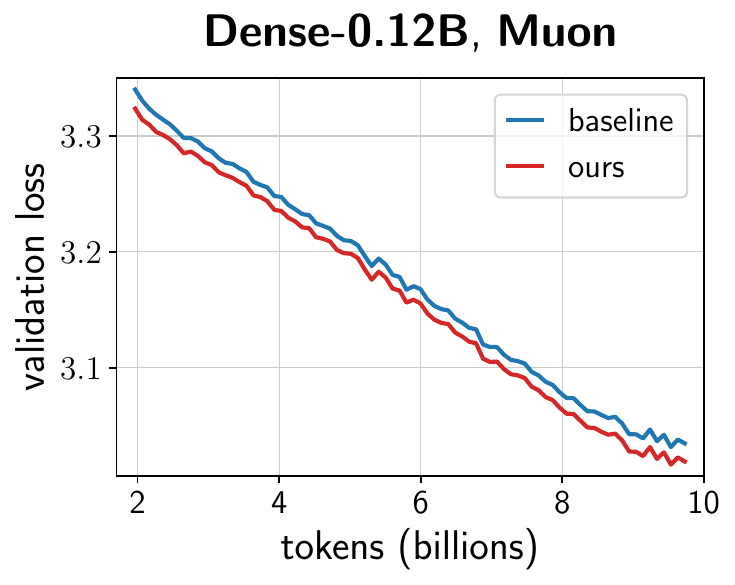}
    \includegraphics[width=0.3\linewidth]{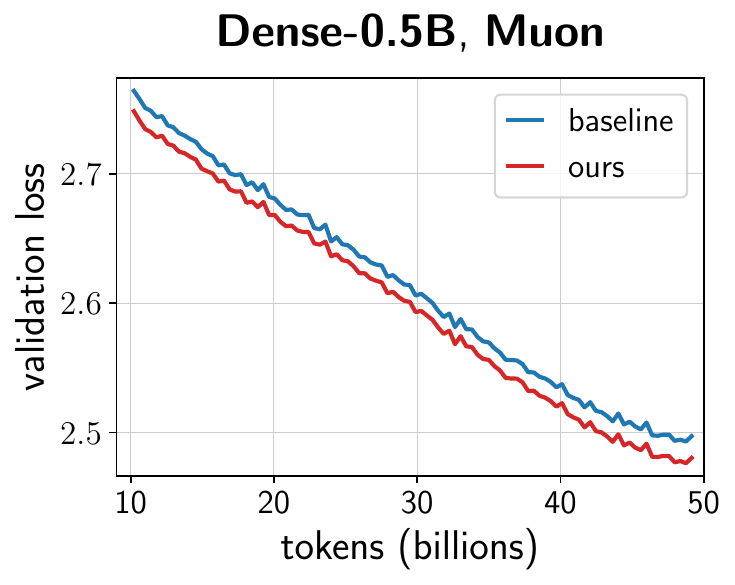}
    \includegraphics[width=0.3\linewidth]{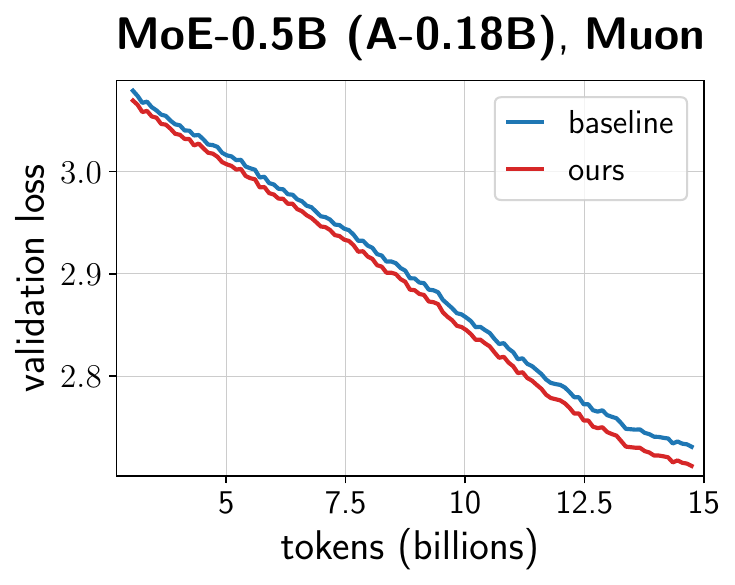}
    % \vspace{-.1cm}
    \caption{Results with the Muon optimizer. Our unified scale-vector strategy consistently improves over the corresponding baselines.}
    % \vspace{-.2cm}
\label{fig: muon}
\end{figure}

\textbf{Muon}~\citep{jordan2024muon,liu2025muon} has recently demonstrated strong efficiency and scalability in LLM pre-training. For the Muon baseline, we adopt the implementation of~\citet{liu2025muon}.   
We conduct experiments on Dense-0.12B, Dense-0.5B, and MoE-0.5B models. The results are shown in the two left panels of Figure~\ref{fig: muon}. 
Compared with the strong Muon baseline, our strategy still reduces the terminal loss by \textit{more than $0.015$} ($0.0156$ for Dense-0.12B, $0.0167$ for Dense-0.5B, and $0.0187$ for MoE-0.5B) and maintains its advantage throughout training.
We conjecture that this behavior is consistent with the AdamW results because Muon still optimizes scale vectors using AdamW, while our strategy primarily modifies scale vectors.

\begin{wrapfigure}{r}{0.62\textwidth}
    \centering
    \vspace{-.2cm}
    \includegraphics[width=0.3\textwidth]{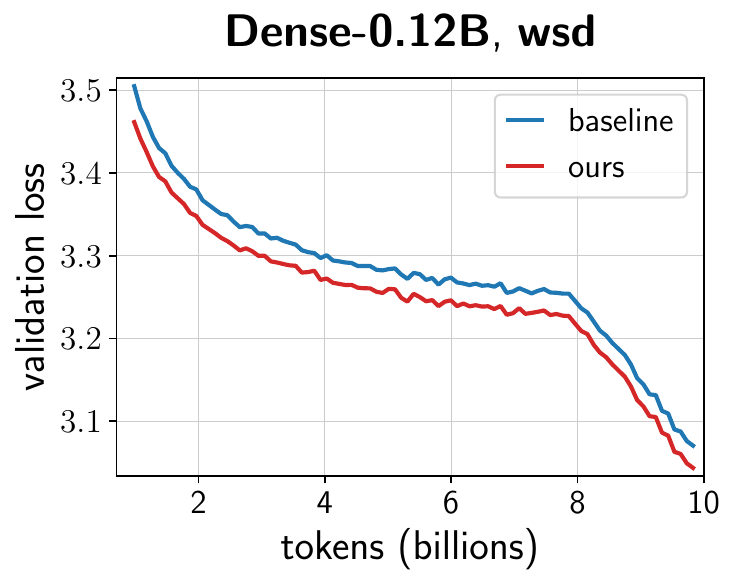}
    \includegraphics[width=0.3\textwidth]{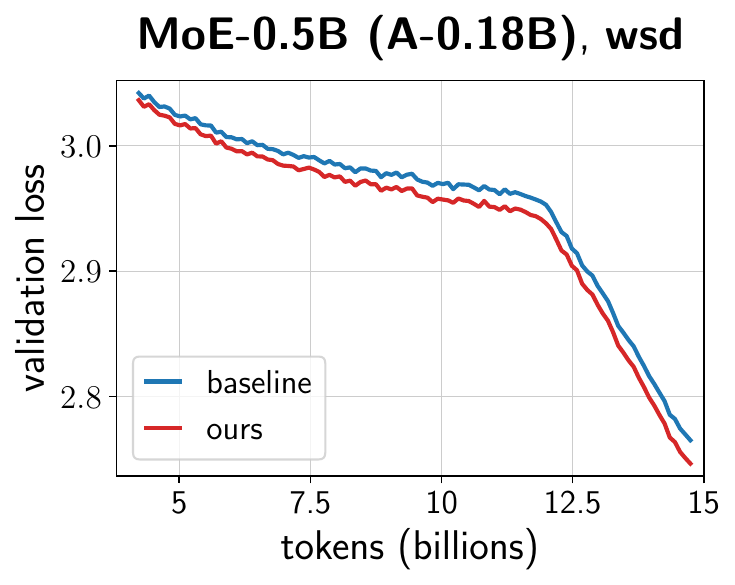}
    % \vspace{-.1cm}
    \caption{Results with the wsd lr schedule. Our unified scale-vector strategy consistently improves over the corresponding baselines.}
    % \vspace{-.2cm}
\label{fig: wsd}
\end{wrapfigure}
The \textbf{wsd} schedule  has recently become a popular choice of LLM pre-training. We also conduct experiments on Dense-0.12B and MoE-0.5B models. The results are shown in Figure~\ref{fig: wsd}. Notably, the advantage of our strategy gradually increases throughout the LR stable phase and does not diminish in the decay phase. This suggests that our strategy may be particularly suitable for modern over-training pipelines that adopt wsd schedules.

\subsection{Overhead Analysis}
\label{subsection: ablation}

We evaluate the parameter and computational overhead of our scale-vector strategy. Experimental details are provided in Appendix~\ref{appendix: subsection: validate proposed method}.

\begin{wrapfigure}{r}{0.25\textwidth}
     \vspace{-.5cm}
     \includegraphics[width=0.245\textwidth]{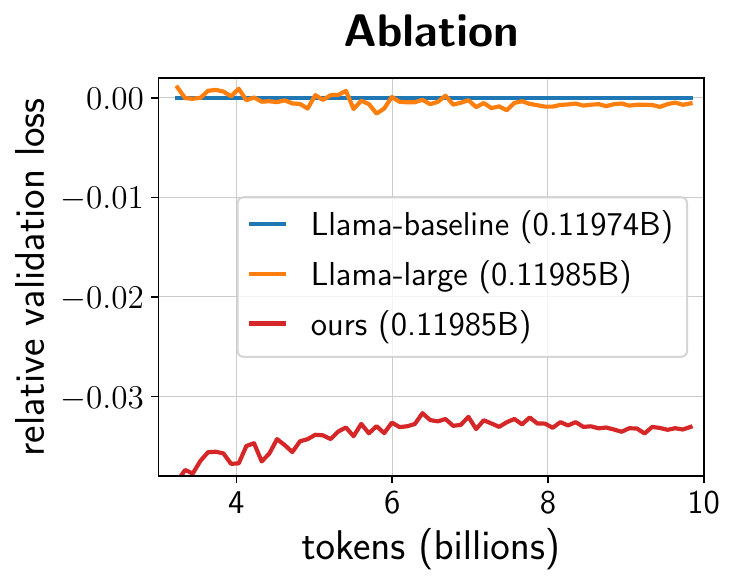}
    \vspace{-.2cm}
    \caption{Ablation on parameter count.}
    \vspace{-.4cm}
    \label{fig: ablation model size}
\end{wrapfigure}
\textbf{Parameter overhead.} 
Our strategies add only $\cO(d)$ scale-vector parameters, whereas Transformer matrix parameters scale as $\cO(d^2)$; thus, the overhead is \textit{negligible}. 
Nevertheless, we conduct a controlled ablation to isolate the effect of parameter count. 
Specifically, we compare against a widened Llama-0.12B baseline, denoted Llama-large, whose parameter count matches our largest variant, the unified strategy \textbf{HG}+\textbf{DNP}+\textbf{OR}+\textbf{IWD} (0.11985B vs. 0.11974B for the Llama-0.12B baseline). 
As shown in Figure~\ref{fig: ablation model size}, our method reduces terminal loss by 0.033, whereas Llama-large yields only a marginal gain. 
The parameter overhead further decreases at larger scales; for example, on Llama-1B, our method increases the parameter count by only $7.85\times10^{-5}$.

\begin{wraptable}{r}{0.55\textwidth}
    \centering
    \caption{Runtime and memory overhead on Dense-1B model. Ratios are relative to the corresponding baseline.}
    \setlength{\tabcolsep}{2.5pt}
    \small
    \begin{tabular}{c|c|c}
    \hline\hline
        model & wall-clock time (ms) & memory usage (MiB) \\ \hline
        baseline & 151 & 34263 \\
        our strategy   & 157 ($1.04\times$) & 34729 ($1.01\times$) \\
    \hline\hline
    \end{tabular}
    \label{tab: overhead}
\end{wraptable}
The \textbf{computational overhead} is also negligible: our strategies introduce only elementwise operations, which scales as $\cO(d)$, compared with the $\cO(d^2)$ cost of Transformer linear projections.
To empirically measure computational overhead, we train a Llama-1B model with standard \texttt{torch.compile}. As shown in Table~\ref{tab: overhead}, our unified strategy incurs only minor increases in wall-clock time and memory usage. Results are averaged over 10 training steps after warmup.

\section{Conclusion}
\label{section: conclusion}

We present a systematic study of scale vectors in LLMs. Although scale vectors account for only a negligible fraction of model parameters, we show that they have a substantial impact on LLM pre-training.
Our theory attributes this effect to optimization rather than expressivity, clarifies when weight decay should be applied, and motivates improved designs based on heterogeneity, placement, and reparameterization. 
Extensive experiments confirm that these designs consistently improve pre-training performance. 
Future work includes extending our analysis to broader architectures and studying how scale-vector design should evolve with model scale.

\section*{Acknowledgment}

We thank Guang Shi and Prof. Weinan E for helpful discussions.

\clearpage

% \bibliographystyle{plainnat}
% \bibliography{ref}

\clearpage

\beginappendix

% \newpage
% \appendix

% \begin{center}
%     \noindent\rule{\textwidth}{4pt} \vspace{-0.2cm}
%     % \noindent\rule{\textwidth}{1.2pt} \vspace{-0.25cm}
%     \LARGE \textbf{Appendix} % \\ ~\\[-0.5cm]
%     \noindent\rule{\textwidth}{1.2pt}
% \end{center}

\startcontents[sections]
\printcontents[sections]{l}{1}{\setcounter{tocdepth}{2}}

\vspace{1.cm}

\section{Related Works}
\label{appendix: related works}

\textbf{Normalization operations.}
The Normalization layer has long been recognized as a key ingredient for stable and efficient deep network training. 
Early work introduced Batch Normalization (BN)~\citep{ioffe2015batch}, which uses batch statistics and achieved great success in computer vision.
Layer Normalization (LN)~\citep{ba2016layer} later removed the dependence on batch statistics and became the standard choice for sequence modeling and Transformers. 
More recently, Root Mean Square Layer Normalization (RMSNorm)~\citep{zhang2019root} further simplified LN by discarding mean centering, and has become the dominant normalization operation in modern LLMs due to its simplicity and efficiency. 
Beyond these standard designs, recent work has also explored alternatives to explicit normalization, such as Dynamic Tanh (DyT)~\citep{zhu2025transformers}, which replaces normalization layers with element-wise nonlinear transformations. 
Overall, prior work has focused primarily on the normalization operation itself, seeking better statistics, lower overhead, or normalization-free replacements. \textit{In contrast}, our work focuses on the scale vector, a component retained across these designs but receiving much less theoretical and empirical attention.

\textbf{Location of normalization layers.}
Another line of work studies \textit{where} normalization should be placed within Transformer blocks. 
The original Transformer adopted Post-Norm, applying normalization after each residual addition. Subsequent work showed Pre-Norm improves training stability, which led it to become the default design in many modern LLMs~\citep{xiong2020layer}. 
However, Pre-Norm also introduces structural limitations, including weaker gradient propagation and reduced contribution from deeper layers~\citep{nguyen2019transformers}. 
Later studies explored combining the advantages of Pre-Norm and Post-Norm, aiming to stable optimization with loss gradient propagation. These works include ~\citep{wang2024deepnet,zhuo2025hybridnorm,chen2026post}.
\textit{In contrast} to this line, which focuses on the location of normalization layers, our work studies the design and mechanism of the scale vectors within these layers.

\textbf{Normalization layers with enhanced nonlinearity.}
More recent work has enriched normalization layers by introducing stronger nonlinear or adaptive modulation mechanisms, for example through gating modules or shallow feedforward networks~\citep{cai2025seednorm,qiu2026unified}. These methods increase the flexibility of normalization layers by making them more adaptive to different inputs.
\textit{In contrast}, rather than modifying normalization layers through additional nonlinear modules, we study the default scale vector itself, provide a theoretical understanding of its necessity and weight decay, and propose several simple but effective improvements. Moreover, several of our designs, including type-dependent weight decay, heterogeneous scale vectors, and dual-sided placements, are largely orthogonal to these methods and can potentially be combined with them.

\vspace{1.cm}

\section{Experimental Details}
\label{appendix: experiments}

All experiments are conducted on H100 80G GPUs.

\subsection{Experimental details for Section~\ref{section: experiments}}
\label{appendix: section: experiments}

\subsubsection{Section~\ref{subsection: main results of unified strategy}}
\label{appendix: subsection: main results of unified strategy}

\textbf{Models.} We utilize two popular classes of LLM models for our pre-training experiments:
 \begin{itemize}    
    \item \textbf{Llama.} 
    Llama~\citep{touvron2023llama} is a dense decoder-only Transformer architecture that uses Rotary Positional Encoding (RoPE)~\citep{su2024roformer}, SwiGLU,  RMSNorm, and a Pre-Norm design. We set $d_{\rm FFN}={\rm int}(8d_{\text{model}}/3)$.
    We pre-train Llama models ranging from 0.12B to 1B parameters. 
    Detailed configurations are provided in Table~\ref{table: dense model config and max lrs}.
    
    \item \textbf{LlamaMoE.} 
    LlamaMoE is a decoder-only mixture-of-experts architecture based on Llama. 
    Each model uses 32 sparse experts, activates 4 sparse experts per token, and includes one shared expert. 
    Following QwenMoE~\citep{yang2024qwen2technicalreport}, the hidden dimension of the shared expert is ${\rm int}(8d_{\rm model}/3)$, whereas that of each
    sparse expert is ${\rm int}(2d_{\rm model}/3)$.
    We pre-train LlamaMoE models ranging from 0.25B to 2B parameters. 
    Detailed configurations are provided in Table~\ref{table: moe model config and max lrs}.
\end{itemize}

\textbf{Token Budget.}
In this section, the total number of training tokens is set to approximately 100 times the number of model parameters for dense models, or 100 times the number of activated parameters for MoE models. This budget substantially exceeds the Chinchilla-optimal regime~\citep{hoffmann2022training} and is closer to industrial pre-training practice. 
As shown in Tables~\ref{table: dense model config and max lrs} and~\ref{table: moe model config and max lrs}, small-scale models use a sequence length of 1,024 and a batch size of 512, corresponding to a total batch size of approximately 0.5M tokens; large-scale models use a sequence length of 4,096 with a batch size of 512, corresponding to approximately 2M tokens per batch.

\textbf{Optimization.} In this section, we use AdamW as the default baseline optimizer~\citep{kingma2014adam,loshchilov2017decoupled}. Following standard Llama pre-training practice~\citep{touvron2023llama}, we set $\beta_1=0.9$, $\beta_2=0.95$, weight decay $\lambda=0.1$, and the gradient clipping threshold to $1.0$. We use the \texttt{cos} schedule: linear warm-up to a peak learning rate \texttt{lr\_max}, followed by cosine decay to \texttt{lr\_min}=\texttt{lr\_max}/20. 
For each setting, we first tune \texttt{lr\_max} for the baseline model and then use the tuned value for both the baseline and the corresponding improved model. The grid search for \texttt{lr\_max} is performed over $\{$\texttt{6e-4}, \texttt{1e-3}, \texttt{2e-3}, \texttt{3e-3}, \texttt{4e-3}, \texttt{5e-3}, \texttt{6e-3}$\}$.
The selected peak lr's are reported in Tables~\ref{table: dense model config and max lrs} and~\ref{table: moe model config and max lrs}.

\begin{table}[!ht]
	% \vspace{-20pt}
		\centering
		%\vspace{-18pt}
        \renewcommand{\arraystretch}{1.25}
		\caption{Model configurations and tuned peak lr's for dense models}
		\label{table: dense model config and max lrs}
		\begin{small}
	 % \addtolength{\tabcolsep}{-3pt} 
		\begin{tabular}{l|c|c|c|c|c|c}
		\hline 
		Acronym & Size & $d_{\text{model}}$ & $n_{\text{head}}$ & $n_{\text{layer}}$ & total batch size & \texttt{lr\_max} \\\hline\hline 
	Llama-0.12B & 0.12B & 768 & 12 & 6 & $1024\times480$ & 4e-3 \\
    Llama-0.25B & 0.25B & 1024 & 16 & 12 & $1024\times480$ & 2e-3 \\
	Llama-0.5B & 0.48B & 1280 & 20 & 18 & $4096\times480$  & 2e-3 \\
    Llama-0.75B & 0.75B & 1536 & 24 & 21 & $4096\times480$ & 2e-3 \\
	Llama-1B & 1.03B & 1792 & 28 & 22 & $4096\times480$ & 1e-3 \\\hline 
	\end{tabular}
	\end{small}
		% \vspace{-5pt}
\end{table}

\begin{table}[!ht]
	% \vspace{-20pt}
		\centering
		%\vspace{-18pt}
        \renewcommand{\arraystretch}{1.25}
		\caption{Model configurations and tuned peak lr's for MoE models.}
		\label{table: moe model config and max lrs}
		\begin{small}
	 % \addtolength{\tabcolsep}{-3pt} 
		\begin{tabular}{l|c|c|c|c|c|c|c}
		\hline 
		Acronym & Size & Activated Size & $d_{\text{model}}$ & $n_{\text{head}}$ & $n_{\text{layer}}$ & total batch size & \texttt{lr\_max} \\
		\hline\hline 
    MoE-0.25B & 0.25B & 0.11B & 640 & 10 & 6  & $1024\times480$ & 3e-3 \\
    MoE-0.5B & 0.48B & 0.18B & 748 & 12 & 9  & $1024\times480$ & 2e-3 \\
	MoE-0.8B & 0.80B & 0.28B & 960 & 15 & 10 & $4096\times480$  & 2e-3 \\
	MoE-1B & 1.06B & 0.38B & 1024 & 16 & 12 & $4096\times480$  & 1e-3 \\
	MoE-2B & 2.00B & 0.62B & 1280 & 20 & 15 & $4096\times480$  & 6e-4 \\
	 \hline 
		\end{tabular}
		\end{small}
		% \vspace{-5pt}
\end{table}

% Due to space limitations in the main text, the results for Dense-0.75B and MoE-0.8B models are reported in Figure~\ref{fig: appendix: dense-0.75B, MoE-0.8B}.

% \begin{figure}[!htp]
%     \centering
%     \includegraphics[width=0.25\linewidth]{figures/unified_strategy/dense_0.75B_adamw.pdf}
%     \includegraphics[width=0.25\linewidth]{figures/unified_strategy/moe_0.8B_adamw.pdf}
%     \caption{Results for Dense-0.75B; MoE-0.8B, complementary to Figures~\ref{fig: dense: unified strategy} and~\ref{fig: moe: unified strategy}.}
%     \label{fig: appendix: dense-0.75B, MoE-0.8B}
% \end{figure}

\subsubsection{Section~\ref{subsection: muon and wsd}}
\label{appendix: subsection: muon and wsd}

\textbf{Muon optimizer.}
We follow the Muon setup of~\citet{jordan2024muon}: Muon is applied only to 2D
matrix blocks in Transformer layers, while AdamW is used for all other parameters,
including scale vectors, embedding layer, and the
output layer. 
Additionally, following~\citet{liu2025muon}, we further use: (i) per-parameter update scaling, with learning-rate multiplier $c=0.2\sqrt{\max\{m,n\}}$ for Muon blocks of shape $\mathbb{R}^{m\times n}$, so that the update RMS norm matches that of AdamW; (ii) Nesterov momentum with coefficient $\theta_{\text{muon}}=0.95$; (iii) weight decay $\lambda=0.1$. 
For Dense-0.12B and MoE-0.5B, we use the batch sizes reported in
Tables~\ref{table: dense model config and max lrs} and
\ref{table: moe model config and max lrs}. To ensure strong baselines, we tune the
peak lr for Muon. The selected \texttt{lr\_max} is \texttt{5e-3} for
Dense-0.12B and \texttt{4e-3} for MoE-0.5B, both of which are larger than the
corresponding tuned values for AdamW.

\textbf{Wsd schedule.}
For AdamW with the \texttt{wsd} schedule~\citep{zhai2022scaling,hu2024minicpm}, we use a linear warmup to the peak learning
rate \texttt{lr\_max}, followed by a stable phase in which the learning rate remains
at \texttt{lr\_max} until $80\%$ of the total training steps, and finally a linear
decay to zero. For fairness, we tune the peak learning rate under the \texttt{wsd}
schedule. The selected \texttt{lr\_max} is \texttt{3e-3} for Dense-0.12B and
\texttt{2e-3} for MoE-0.5B.

\subsubsection{Section~\ref{subsection: validate proposed method} and~\ref{subsection: ablation}}
\label{appendix: subsection: validate proposed method}

The baseline model in Figure~\ref{fig: validation of each strategy} uses the same training configuration as Llama-0.12B in
Appendix~\ref{appendix: subsection: main results of unified strategy}.

For the model-size ablation (Figure~\ref{fig: ablation model size}), we construct a larger Llama variant based on Llama-0.12B.
Directly increasing the hidden dimension would require the dimension to remain compatible with both RoPE and the attention heads, leading to a parameter count much larger than the target. 
Therefore, we instead increase the FFN hidden dimension from $d_{\rm FFN}$ to $d_{\rm FFN}+8$, which closely matches the target parameter count.
We retune the learning rate for this larger model, and the selected value remains \texttt{4e-3}.

The experiments in Table~\ref{tab: overhead} use the same training configuration as Llama-1B in Appendix~\ref{appendix: subsection: main results of unified strategy}.

\subsection{Experimental details for Section~\ref{section: understanding} and~\ref{section: improving}}

\subsubsection{Section~\ref{subsection: understanding necessity}}
\label{appendix: subsection: understanding necessity}

The baseline model (w/ scale vectors) in Figure~\ref{fig: w/o w/ scale vectors} uses the same training configuration as Llama-0.12B in Appendix~\ref{appendix: subsection: main results of unified strategy}.

The comparison model (w/o scale vectors) in Figure~\ref{fig: w/o w/ scale vectors} is obtained by removing all scale vectors from Llama-0.12B. 
For the peak lr \texttt{max\_lr}, the left subfigure uses the same value as the baseline, \texttt{4e-3}, whereas the right subfigure retunes this value and selects \texttt{2e-3}. All other settings follow
Appendix~\ref{appendix: subsection: main results of unified strategy}.

The longer-training experiments only change the number of training steps. Specifically, the model w/o scale vectors is trained for $1.4\times$ more steps in the left subfigure and $1.2\times$ more steps in the right subfigure.

\subsubsection{Section~\ref{subsection: understanding weight decay}}
\label{appendix: subsection: understanding weight decay}

For Figure~\ref{fig: understanding: weight decay}, we train Gemma-0.5B
models~\citep{gemmateam2025gemma3technicalreport} on 
the same pre-training corpus as in~\ref{section: experiments}
% FineWeb-Edu 
for approximately 10B and 50B tokens using the Muon optimizer~\citep{jordan2024muon}.

\begin{itemize}

    \item \textbf{Model.}
    Gemma-0.5B uses the same $d_{\rm model}$, $d_{\rm FFN}$, $n_{\rm head}$, and $n_{\rm layer}$ as Llama-0.5B in Appendix~\ref{appendix: subsection: main results of unified strategy}. 
    However, its normalization layers differ from those of Llama. We compare the two architectures in Figure~\ref{fig: llama, gemma}.
    
    \item \textbf{Training.} 
    The Muon baseline follows the setup in Appendix~\ref{appendix: subsection: muon and wsd}. 
    We use the same batch size as Llama-0.5B in Table~\ref{table: dense model config and max lrs}. 
    We train for 5,000 steps, corresponding to approximately 10B tokens, denoted by 10BT, and for 25,000 steps, corresponding to approximately 50B tokens, denoted by 50BT.
    We first tune the peak lr for the baseline setting, which applies weight decay to all scale vectors, under the 50BT training budget. 
    The selected \texttt{max\_lr} is \texttt{4e-3}, which is then used for all other experiments.    
\end{itemize}

\subsubsection{Experimental details for Section~\ref{subsection: heterogeneity}}
\label{appendix: subsection: heterogeneity}

The model in Figure~\ref{fig: HG: qkv dynamics} uses the same training configuration as Llama-0.12B in Appendix~\ref{appendix: subsection: main results of unified strategy}.

\vspace{1.cm}

\section{Proofs in Section~\ref{section: understanding}}
\label{appendix: proof of understanding}

\subsection{Proofs in Section~\ref{subsection: understanding necessity}}
\label{appendix: proof of understanding necessity}

\begin{proof}[Proof of Theorem~\ref{thm: optimization advantage w/ scale vector}]
For clarity, we denote
\begin{align*}
    \bW:=\bW_g,\quad\bU:=\bW_f.
\end{align*}

Let $\bz=\Norm(\bx)=\sqrt d\,\frac{\bx}{\|\bx\|_2}$.
Since $\bx\sim \mathcal N(\bzero,\bI_d)$ is isotropic, $\bz$ is uniformly
distributed on the sphere of radius $\sqrt d$. Hence
$\mathbb E[\bz\bz^\top]=\bI_d$.
Therefore,
\[
    \mathcal L_g(\bW)
    =
    \frac12\|\bW-\bW^\star\|_F^2.
\]
Similarly, writing
\[  f(\bx;\bU,\bgamma)=\bU\bD_{\bgamma}\bz,\qquad\bD_{\bgamma}=\operatorname{diag}(\bgamma)\in\bbR^{d\times d},
\]
we have
\[
    \mathcal L_f(\bU,\bgamma)
    =
    \frac12\|\bU\bD_{\bgamma}-\bW^\star\|_F^2.
\]

\textit{Step I: dynamics of the model without scale vectors.}
Since
\[
    \mathcal L_g(\bW)=\frac12\|\bW-\bW^\star\|_F^2,
\]
its gradient flow is
\[
    \frac{\rd\bW}{\rd t}=\bW^\star-\bW.
\]
With $\bW(0)=\bzero$, this gives
\[
    \bW(t)=(1-e^{-t})\bW^\star.
\]
Thus
\[
    \mathcal L_g(\bW(t))
    =
    \frac12 e^{-2t}\|\bW^\star\|_F^2.
\]

\textit{Step II: dynamics of the model with scale vectors.} Write the columns of $\bU$ and
$\bW^\star$ as
\[
    \bU=[\bu_1,\dots,\bu_d],
    \qquad
    \bW^\star=[\bw_1^\star,\dots,\bw_d^\star],
\]
where $\bu_j,\bw_j^\star\in\mathbb R^c$. Then
\[
    \bU\bD_{\bgamma}
    =
    [\gamma_1\bu_1,\dots,\gamma_d\bu_d],
\]
and the loss decomposes as
\[
    \mathcal L_f(\bU,\bgamma)
    =
    \frac12\sum_{j=1}^d
    \|\gamma_j\bu_j-\bw_j^\star\|_2^2.
\]
For each coordinate $j$, define
\[
    \br_j=\bw_j^\star-\gamma_j\bu_j,
    \qquad
    \ell_j=\frac12\|\br_j\|_2^2.
\]
Then
\[
    \mathcal L_f(\bU,\bgamma)=\sum_{j=1}^d \ell_j.
\]

The gradient-flow equations for the $j$-th column are
\[
    \frac{\rd\bu_j}{\rd t}=\gamma_j\br_j,
    \qquad
    \frac{\rd\gamma_j}{\rd t}=\bu_j^\top \br_j.
\]
Therefore,
\[
    \frac{\rd}{\rd t}(\gamma_j\bu_j)
    =
    \frac{\rd\gamma_j}{\rd t}\bu_j+\gamma_j\frac{\rd\bu_j}{\rd t}
    =
    \bu_j\bu_j^\top \br_j+\gamma_j^2\br_j.
\]
Equivalently,
\[
    \frac{\rd\br_j}{\rd t}=-(\gamma_j^2\bI_c+\bu_j\bu_j^\top)\br_j.
\]
Hence
\[
    \frac{\rd\ell_j}{\rd t}
    =
    \br_j^\top\frac{\rd\br_j}{\rd t}
    =
    -\gamma_j^2\|\br_j\|_2^2
    -
    (\bu_j^\top\br_j)^2
    =
    -2\gamma_j^2\ell_j-(\bu_j^\top\br_j)^2.
\]

We also have the conservation law
\[
    \frac{\rd}{\rd t}
    \left(
        \gamma_j^2-\|\bu_j\|_2^2
    \right)
    =
    2\gamma_j\frac{\rd\gamma_j}{\rd t}-2\bu_j^\top\frac{\rd\bu_j}{\rd t}
    =
    2\gamma_j\bu_j^\top\br_j
    -
    2\bu_j^\top(\gamma_j\br_j)
    =
    0.
\]
Since $\gamma_j(0)=1$ and $\bu_j(0)=\bzero$, it follows that
\[
    \gamma_j(t)^2-\|\bu_j(t)\|_2^2=1.
\]
Thus
\begin{equation}\label{eq: proof: thm: w/ SV: conservation}
    \gamma_j(t)^2=1+\|\bu_j(t)\|_2^2\ge 1.
\end{equation}
Consequently,
\[
    \frac{\rd\ell_j}{\rd t}
    =
    -2\gamma_j^2\ell_j-(\bu_j^\top\br_j)^2
    \le
    -2\ell_j.
\]
Equivalently,
\begin{equation}\label{eq: proof: thm: w/ SV: core}
    \frac{\rd}{\rd t}
    \left(
        e^{2t}\ell_j(t)
    \right)
    =
    e^{2t}
    \left(
        \frac{\rd}{\rd t}\ell_j(t)+2\ell_j(t)
    \right)
    =
    -e^{2t}
    \left(
        2(\gamma_j(t)^2-1)\ell_j(t)
        +
        (\bu_j(t)^\top\br_j(t))^2
    \right)
    \le 0.
\end{equation}
Therefore,
\[
    \ell_j(t)\le e^{-2t}\ell_j(0).
\]

\textit{Step III: comparison of the errors.}
We now prove strict inequality whenever $\bw_j^\star\neq \bzero$. For such a
column,
\[
    \bu_j(0)=\bzero,
    \qquad
    \br_j(0)=\bw_j^\star,
\]
and hence
\[
    \frac{\rd}{\rd t}\bu_j(0)
    =
    \gamma_j(0)\br_j(0)
    =
    \bw_j^\star.
\]
Thus, for all sufficiently small $s>0$,
\[
    \bu_j(s)=s\bw_j^\star+\cO(s^2),
    \qquad
    \br_j(s)=\bw_j^\star+\cO(s).
\]
Therefore, for all sufficiently small $s>0$,
\[
    \bu_j(s)^\top\br_j(s)
    =
    s\|\bw_j^\star\|_2^2+\cO(s^2)>0.
\]
Recalling~\eqref{eq: proof: thm: w/ SV: conservation}, $\gamma_j(s)^2\geq1$. Therefore, the integrand
\[2(\gamma_j(s)^2-1)\ell_j(s)+(\bu_j(s)^\top\br_j(s))^2\]
in~\eqref{eq: proof: thm: w/ SV: core} is strictly positive on a nontrivial interval contained in $(0,t)$ for every
$t>0$. Consequently,
\[
    \ell_j(t)<e^{-2t}\ell_j(0),
    \qquad \forall t>0.
\]
For the model $g$, the corresponding column-wise loss is exactly
\[
    \ell_{g,j}(t)
    =
    \frac12 e^{-2t}\|\bw_j^\star\|_2^2
    =
    e^{-2t}\ell_j(0).
\]
Hence
\[
    \ell_j(t)\le \ell_{g,j}(t)
\]
for every $j$, and the inequality is strict for every column satisfying
$\bw_j^\star\neq \bzero$. Since $\bW^\star\neq \bzero$, at least one such
column exists. Summing over all columns gives
\[
    \mathcal L_f(\bU(t),\bgamma(t))
    <
    \mathcal L_g(\bW(t)),
    \qquad \forall t>0.
\]
This proves the theorem.
\end{proof}

\textbf{Comparison with}~\citet{arora2018optimization}. 
A related line of work studies implicit acceleration by over-parameterization~\citep{arora2018optimization}. It shows that deep linear networks (DLNs), \( \bW_L\cdots\bW_1\bz \), can accelerate optimization by inducing an implicit preconditioner, without increasing the expressivity of standard linear model. Our setting and results differs in three key aspects. 
\begin{itemize}
    \item \textbf{Parameter overhead.} 
    Scale vectors introduce substantially smaller overhead. A DLN $f_{\text{DLN}}(\bz;\btheta)=\bW_L\cdots\bW_1\bz$ replaces a single $d_0\cdot d_L$ linear map $\bW\bz$ with
    $\sum_{\ell=1}^{L} d_{\ell-1}\cdot d_\ell$ matrix parameters. 
    In contrast, adding a standard input-side scale vector, $\bW(\bgamma\odot\bz)$, introduces only $d_0$ additional parameters, which is negligible relative to the $d_0\cdot d_L$ matrix parameters in linear map. It is therefore notable that such a lightweight parameterization can still improve optimization.

    \item \textbf{Preconditioner structure.}
    The induced preconditioners have different structures. DLNs induce a global spectral preconditioner aligned with the singular directions of the matrix. 
    In contrast, scale vectors induce a channel-wise preconditioner: each input channel has its own preconditioner, which acts directly on the corresponding dynamics. 
     Comparing the regimes in which these two preconditioners are preferable is an interesting direction for future work.

    \item \textbf{Initialization and early-stage dynamics.}
    The two parameterizations also differ in their initialization behavior. Scale vectors are typically initialized as $\bone$, so adding them preserves the initial function. By contrast, DLNs introduce additional matrix factors; under standard small random initialization, the initial end-to-end differs from the baseline unless explicitly matched, and DLNs can be slower than the linear baseline model at early phase, as shown in~\citep{arora2018optimization}. 
    In contrast, scale vectors can accelerate training from the beginning, whereas DLN-induced acceleration may emerge only after the end-to-end matrix has grown sufficiently.
    
\end{itemize}

\subsection{Proofs in Section \ref{subsection: understanding weight decay}}

\subsubsection{Proof of Theorem~\ref{thm: wd, pre-norm}}
\label{appendix:wd-pf}

Note that
\[
    \mathbb{E}[\Norm(\bx)\Norm(\bx)^\top]=\bI_d.
\]
For notational convenience and without loss of generality, we consider the
problem
\begin{equation}
    \cL(\bw,\bgamma)=L(\ba)
    =
    \frac{1}{2}\|\ba-\ba_\star\|_2^2,
    \quad
    \ba_t=\bgamma_t\odot\bw_t\in\mathbb{R}^d,
    \quad\ba_\star=\bw^\star.
\end{equation}
Here \(\ba_t\) is the effective parameter, while \(\bw_t\) and \(\bgamma_t\)
are the two factors in the scale-vector parameterization.

\textbf{Stochastic differential equation (SDE).} The continuous-time modeling of Stochastic gradient descent (SGD) can be written as follows~\citep{li2017stochastic}.
Let $q=\eta\sigma^2$.
Then the SDE is
\begin{align}
    \rd \bw_t
    &=
    -\big(\bgamma_t\odot \bg_t+\lambda \bw_t\big)\rd t
    +
    \sqrt q\,\rd \bB_t^{\bw},
    \label{eq:sde-w}\\
    \rd \bgamma_t
    &=
    -\big(\bw_t\odot \bg_t+\mu\bgamma_t\big)\rd t
    +
    \sqrt q\,\rd \bB_t^{\bgamma},
    \label{eq:sde-gamma}
\end{align}
where
\[
    \bg_t=\nabla_{\ba}L(\ba_t)=\ba_t-\ba_\star,
\]
and $\rd\bB_t^{\bw},\rd\bB_t^{\bgamma}$ denote independent standard Brownian
motions in \(\mathbb R^d\). 
% The parameters are initialized as
% \[
%     \bw_t|_{t=0}=\bw_0,
%     \qquad
%     \bgamma_0=\bone_d.
% \]
% We mainly consider two initialization methods for \(\bw\):
% \[
%     \bw_0=\bzero_d
%     \qquad
%     \text{and}
%     \qquad
%     \bw_0\sim \mathcal N(\bzero,\varsigma^2\bI_d).
% \]

By defining the error vector
\[
    \bepsilon_t=\ba_t-\ba_\star,
\]
the loss becomes
\[
    L_t=\frac12\|\bepsilon_t\|_2^2.
\]

\textbf{Hessian sharpness metrics of the loss.}
% Recall that the effective parameter is
% \[
%     \ba=\bw\odot\bgamma,
% \]
% and the loss is
% \[
%     \mathcal L(\bw,\bgamma)
%     =
%     \frac12\|\ba-\ba_*\|_2^2
%     =
%     \frac12\|\bw\odot\bgamma-\ba_*\|_2^2 .
% \]
Let
\[
    \bg=\nabla_{\ba} L(\ba)=\ba-\ba_\star,
\]
and define
\[
    \bD_{\bw}=\operatorname{diag}(\bw),
    \qquad
    \bD_{\bgamma}=\operatorname{diag}(\bgamma),
    \qquad
    \bD_{\bg}=\operatorname{diag}(\bg).
\]
Then the Hessian of \(\mathcal L\) with respect to the parameters
\((\bw,\bgamma)\) is
\begin{equation}
    \nabla^2_{(\bw,\bgamma)}\mathcal L
    =
    \begin{pmatrix}
        \bD_{\bgamma}^2
        &
        \bD_{\bgamma}\bD_{\bw}+\bD_{\bg}
        \\
        \bD_{\bw}\bD_{\bgamma}+\bD_{\bg}
        &
        \bD_{\bw}^2
    \end{pmatrix}.
    \label{eq:hessian-loss-block}
\end{equation}
Since
\[
    \bD_{\bgamma}\bD_{\bw}+\bD_{\bg}
    =
    \operatorname{diag}(2\bgamma\odot\bw-\ba_*),
\]
the Hessian is block diagonal across coordinates. The \(i\)-th block is
\begin{equation}
    \bH_i
    =
    \begin{pmatrix}
        \gamma_i^2
        &
        2\gamma_iw_i-a_{*,i}
        \\
        2\gamma_iw_i-a_{*,i}
        &
        w_i^2
    \end{pmatrix}.
    \label{eq:hessian-block-i}
\end{equation}
Therefore,
\[
    \nabla^2_{(\bw,\bgamma)}\mathcal L
    =
    \operatorname{diag}(\bH_1,\ldots,\bH_d).
\]

The two eigenvalues of \(\bH_i\) are
\begin{equation}
    \lambda_{i,\pm}
    =
    \frac12
    \left[
        \gamma_i^2+w_i^2
        \pm
        \sqrt{
            (\gamma_i^2-w_i^2)^2
            +
            4(2\gamma_iw_i-a_{*,i})^2
        }
    \right].
    \label{eq:hessian-block-eigs}
\end{equation}
Consequently,
\begin{equation}
    \lambda_{\max}
    \left(
        \nabla^2\mathcal L(\bw,\bgamma)
    \right)
    =
    \max_{1\le i\le d}\lambda_{i,+}.
    \label{eq:hessian-maxeig}
\end{equation}
% and
% \begin{equation}
%     \lambda_{\min}
%     \left(
%         \nabla^2\mathcal L(\bw,\bgamma)
%     \right)
%     =
%     \min_{1\le i\le d}\lambda_{i,-}.
%     \label{eq:hessian-mineig}
% \end{equation}

The trace is
\begin{equation}
    \operatorname{Tr}
    \left(
        \nabla^2\mathcal L(\bw,\bgamma)
    \right)
    =
    \sum_{i=1}^d
    \left(
        \gamma_i^2+w_i^2
    \right)
    =
    \|\bgamma\|_2^2+\|\bw\|_2^2.
    \label{eq:hessian-trace}
\end{equation}
The squared Frobenius norm is
\begin{equation}
    \left\|
        \nabla^2\mathcal L(\bw,\bgamma)
    \right\|_{\rm F}^2
    =
    \sum_{i=1}^d
    \left[
        \gamma_i^4+w_i^4
        +
        2(2\gamma_iw_i-a_{*,i})^2
    \right].
    \label{eq:hessian-fro}
\end{equation}

Then the formal version of Theorem~\ref{thm: wd, pre-norm} is as follows:

\begin{theorem}[Norm and sharpness Comparison]
\label{thm: complete: weight decay}
Consider \(\lambda>0\). The trajectory of
\eqref{eq:sde-w}\eqref{eq:sde-gamma} satisfies $\sup_{t\ge0}\mathbb E\|\bw_t\|^2<\infty$. Moreover,
\begin{itemize}
\item If \(\mu>0\), then
$\sup_{t\ge0}\mathbb E\|\bgamma_t\|_2^2<\infty$. Consequently, $\sup_{t\ge0}\mathbb E[\lambda_{\max}(\nabla^2\cL(t_k))]<\infty,\ 
\sup_{t\ge0}\mathbb E[\Tr(\nabla^2\cL(t_k))]<\infty,\ 
\sup_{t\ge0}\mathbb E\|\nabla^2\cL(t_k)\|_{\rm F}<\infty$.
\item In contrast, if \(\mu=0\), then
$\sup_{t\ge0}\mathbb E\|\bgamma_t\|^2=\infty$.
Therefore, there exists a sequence \(t_k\to\infty\) such that $\mathbb E\|\bgamma_{t_k}\|^2\to\infty$. 
Consequently, $
    \bbE[\lambda_{\max}(\nabla^2\mathcal L(t_k))]\to\infty,\ 
    \bbE[\Tr(\nabla^2\mathcal L(t_k))]\to\infty,\ 
    \bbE\|\nabla^2\mathcal L(t_k)\|_{\rm F}\to\infty.$
\end{itemize}

\end{theorem}

\begin{proof}
Consider the coordinate-wise scale difference
\[
    s_{t,i}=\gamma_{t,i}^2-w_{t,i}^2.
\]
By It\^o's formula, we have
\begin{equation}
    \rd s_{t,i}
    =
    2\big(\lambda w_{t,i}^2-\mu \gamma_{t,i}^2\big)\rd t
    +
    2\sqrt q
    \big(
        \gamma_{t,i}\rd B_{t,i}^{\bgamma}
        -
        w_{t,i}\rd B_{t,i}^{\bw}
    \big).
    \label{eq:scale-diff-sde}
\end{equation}
Notice that the It\^o correction terms cancel because \(w_{t,i}\) and
\(\gamma_{t,i}\) are both one-dimensional coordinates with the same noise
variance \(q\).

We first show that \(\mathbb E\|\bw_t\|_2^2\) is bounded over \(t\ge0\).
Using It\^o's formula,
\begin{align}
    \frac{\rd}{\rd t}\mathbb E\|\bw_t\|_2^2
    &=
    -2\lambda \mathbb E\|\bw_t\|_2^2
    -
    2\mathbb E
    \left[
        (\bgamma_t\odot\bw_t)^\top
        (\ba_t-\ba_\star)
    \right]
    +
    dq \notag\\
    &=
    -2\lambda \mathbb E\|\bw_t\|_2^2
    -
    2\mathbb E
    \left[
        \ba_t^\top(\ba_t-\ba_\star)
    \right]
    +
    dq.
\end{align}
For any \(\ba\), we have
\[
    -2\ba^\top(\ba-\ba_\star)
    =
    -2\|\ba\|_2^2+2\ba^\top\ba_\star
    \le
    \frac12\|\ba_\star\|_2^2.
\]
Therefore,
\begin{equation}
    \frac{\rd}{\rd t}\mathbb E\|\bw_t\|_2^2
    \le
    -2\lambda \mathbb E\|\bw_t\|_2^2
    +
    \frac12\|\ba_\star\|_2^2
    +
    dq.
\end{equation}
By Gr\"onwall's inequality,
\begin{equation}
    \mathbb E\|\bw_t\|_2^2
    \le
    e^{-2\lambda t}\mathbb E\|\bw_0\|_2^2
    +
    \frac{
        dq+\frac12\|\ba_\star\|_2^2
    }{
        2\lambda
    }
    \left(1-e^{-2\lambda t}\right).
    \label{eq:gr-bdw}
\end{equation}
Thus \(\mathbb E\|\bw_t\|_2^2\) is bounded uniformly over \(t\ge0\).

\begin{itemize}
    \item \textbf{Without weight decay on \(\bgamma\) \((\mu=0)\):}

    Taking expectation in \eqref{eq:scale-diff-sde}, we get
    \[
        \frac{\rd}{\rd t}\mathbb E[s_{t,i}]
        =
        2\lambda \mathbb E[w_{t,i}^2]
        \ge 0.
    \]
    Suppose, for contradiction, that there exists \(M>0\) such that
    \[
        \sup_{t\ge0}\mathbb E\|\bgamma_t\|_2^2\le M.
    \]
    Combining this with
    \[
        \sup_{t\ge0}\mathbb E\|\bw_t\|_2^2<\infty,
    \]
    the occupation measures
    \[
        \nu_T(A)
        =
        \frac1T\int_0^T
        \mathbb P\big((\bw_t,\bgamma_t)\in A\big)\rd t
    \]
    are uniformly tight. By the Krylov--Bogoliubov theorem, there exists an
    invariant probability measure \(\pi\).

    Let \(\mathcal A_0\) denote the generator of the SDE when \(\mu=0\).
    From \eqref{eq:scale-diff-sde},
    \[
        \mathcal A_0(\gamma_i^2-w_i^2)
        =
        2\lambda w_i^2.
    \]
    Since \(\pi\) is invariant,
    \begin{equation}
        0
        =
        \int
        \mathcal A_0(\gamma_i^2-w_i^2)\rd\pi
        =
        2\lambda\int w_i^2\rd\pi.
    \end{equation}
    Since \(\lambda>0\), this implies
    \[
        w_i=0
        \quad
        \pi\text{-a.s.}
    \]
    for every \(i\). Hence
    \[
        \bw=\bzero
        \quad
        \pi\text{-a.s.}
    \]

    Now apply the generator to \(\|\bw\|_2^2\). We have
    \[
        \mathcal A_0\|\bw\|_2^2
        =
        dq
        -
        2\ba^\top(\ba-\ba_\star)
        -
        2\lambda\|\bw\|_2^2.
    \]
    Since \(\bw=\bzero\) \(\pi\)-a.s., we also have
    \[
        \ba=\bgamma\odot\bw=\bzero
        \quad
        \pi\text{-a.s.}
    \]
    Therefore,
    \[
        \mathcal A_0\|\bw\|_2^2=dq
        \quad
        \pi\text{-a.s.}
    \]
    Using invariance again,
    \[
        0
        =
        \int \mathcal A_0\|\bw\|_2^2\rd\pi
        =
        dq,
    \]
    which contradicts \(q>0\). Thus
    \[
        \sup_{t\ge0}\mathbb E\|\bgamma_t\|_2^2=\infty.
    \]

    \item \textbf{With weight decay on \(\bgamma\) \((\mu>0)\):}

    By It\^o's formula,
    \begin{align}
        \frac{\rd}{\rd t}\mathbb E\|\bgamma_t\|_2^2
        &=
        -2\mu\mathbb E\|\bgamma_t\|_2^2
        -
        2\mathbb E
        \left[
            (\bgamma_t\odot\bw_t)^\top
            (\ba_t-\ba_\star)
        \right]
        +
        dq \notag\\
        &=
        -2\mu\mathbb E\|\bgamma_t\|_2^2
        -
        2\mathbb E
        \left[
            \ba_t^\top(\ba_t-\ba_\star)
        \right]
        +
        dq.
    \end{align}
    Using again
    \[
        -2\ba^\top(\ba-\ba_\star)
        \le
        \frac12\|\ba_\star\|_2^2,
    \]
    we obtain
    \[
        \frac{\rd}{\rd t}\mathbb E\|\bgamma_t\|_2^2
        \le
        -2\mu\mathbb E\|\bgamma_t\|_2^2
        +
        \frac12\|\ba_\star\|_2^2
        +
        dq.
    \]
    Hence, by Gr\"onwall's inequality,
    \[
        \mathbb E\|\bgamma_t\|_2^2
        \le
        e^{-2\mu t}\mathbb E\|\bgamma_0\|_2^2
        +
        \frac{
            dq+\frac12\|\ba_\star\|_2^2
        }{
            2\mu
        }
        \left(1-e^{-2\mu t}\right).
    \]
    Therefore,
    \[
        \sup_{t\ge0}\mathbb E\|\bgamma_t\|_2^2<\infty.
    \]
\end{itemize}

Finally, the estimates of Hessian statistics ($\lambda_{\max}(\nabla^2\cL),\Tr(\nabla^2\cL),\|\nabla^2\cL\|_{\rm F}$) follow~\eqref{eq:hessian-maxeig}~\eqref{eq:hessian-trace}~\eqref{eq:hessian-fro} naturally.
\end{proof}

This conclusion is related to that of~\citet{ziyin2024parameter}, who show that weight decay can induce balanced solutions in deep linear networks under stochastic dynamics. The key difference is that their analysis applies the same weight decay to all parameters, whereas our setting distinguishes between the two parameters, allowing us to isolate the effect of applying or omitting weight decay on \(\bgamma\).

\subsubsection{How Hessian sharpness influences loss descent}
\label{appendix:wd-hessian}

We illustrate how Hessian sharpness metrics influence the loss descent of SGD.
We begin by examining the loss
dynamics. Let \(\cL(\cdot)\) be twice differentiable.
Consider the SGD update
\[
    \btheta_{t+1}
    =
    \btheta_t
    -
    \eta\left(\nabla\cL(\btheta_t)+\bxi_t\right),
\]
where the gradient noise satisfies $\bbE[\bxi_t]=\bzero$.
A second-order Taylor expansion gives:
\begin{equation}\label{eqn: x2}
\bbE[\cL(\btheta_{t+1})] = \bbE[\cQ(\btheta_t)]  + \frac{\eta^2}{2} \bbE[\bxi_t^\top \nabla^2\cL(\btheta_t)\bxi_t] + o(\eta^2),
\end{equation}
where 
\begin{align*}
    \cQ(\btheta_t)&:=\cL(\btheta_t)-\eta\|\nabla \cL(\btheta_t)\|^2+\frac{\eta^2}{2}\nabla \cL(\btheta_t)^\top \nabla^2\cL(\btheta_t)\nabla\cL(\btheta_t),\\
    \cP(\btheta_t)&:=\bbE[\bxi_t^\top \nabla^2\cL(\btheta_t)\bxi_t],
\end{align*}
Here \(\cQ(\btheta)\) denotes the deterministic GD contribution, while
\(\cP(\btheta)\) denotes the contribution from SGD noise.

\begin{itemize}
    \item \textbf{GD term.} 
    By the standard descent lemma, the deterministic GD term is controlled by the maximum eigenvalue of the Hessian:
    \begin{align*}
        \cQ(\btheta_t)\leq\cL(\btheta_t)-\eta\left(1-\frac{\eta{\color{red!70!black}\lambda_{\max}(\nabla^2\cL(\btheta_t))}}{2}\right)\|\nabla\cL(\btheta_t)\|^2.
    \end{align*}
    \item \textbf{Noise term.} 
    The noise contribution is $ \cP(\btheta_t)=\bbE[\bxi_t^\top \nabla^2\cL(\btheta_t)\bxi_t]$.
     Therefore, different assumptions on the SGD noise covariance $\bbE[\bxi_t\bxi_t^\top]$ lead to
    different Hessian statistics controlling the noise-induced loss change.

    \begin{itemize}
    \item 
     Classical optimization analyses~\citep{hazan2016introduction} often assume the bounded noise variance: $\bbE[\|\bxi_t\|^2]\leq\sigma^2$, with $\sigma$ being a fixed constant. It follows that 
    \begin{align*}
        \cP(\btheta_t) \leq \sigma^2{\color{red!70!black}\lambda_{\max}(\nabla^2\cL(\btheta_t))}.
        \end{align*}
        
    \item 
    Another common assumption is the affine variance condition~\citep{bottou2018optimization}: $\bbE[\|\bxi_t\|^2]\leq \sigma_0^2+\sigma_1^2\|\nabla \cL(\btheta_t)\|^2$. Then
    \begin{align*}
       \cP(\btheta_t) \leq (\sigma_0^2+\sigma_1^2\|\nabla \cL(\btheta_t)\|^2){\color{red!70!black}\lambda_{\max}(\nabla^2\cL(\btheta_t))}.
    \end{align*}
    
    \item
    For regression with squared loss, several works show that the magnitude of SGD noise can be bounded by the loss value~\citep{feng2021inverse,mori2021logarithmic,wojtowytsch2021stochastic,liu2021noise}: $\bbE[\|  \bxi_t\|^2]\leq \sigma^2\cL(\btheta_t)$. It follows that:
    \begin{align*}
        \cP(\btheta_t) \leq \sigma^2\cL(\btheta_t){\color{red!70!black}\lambda_{\max}(\nabla^2\cL(\btheta_t))}.
    \end{align*}

    \item A common simplification in SDE analyses of SGD is to approximate the gradient noise by isotropic Gaussian noise~\citep{welling2011bayesian,raginsky2017non}: $\bbE[\bxi_t\bxi_t^\top]\simeq\sigma^2\bI_d$. Then
    \begin{align*}
        \cP(\btheta_t) \simeq \sigma^2{\color{red!70!black}\Tr(\nabla^2\cL(\btheta_t))}.
    \end{align*}

    \item In near-interpolation regimes, some works model the SGD noise covariance as approximately aligned with the Hessian~\citep{haochen2021shape,damian2021label,li2021happens}: $\bbE[\bxi_t\bxi_t^\top]\simeq\sigma^2 \nabla^2\cL(\btheta)$. In this case,
    \begin{align*}
       \cP(\btheta_t) \simeq{\color{red!70!black}\|\nabla^2\cL(\btheta_t)\|_{\rm F}^2}.
    \end{align*} 

    \item For squared loss in near-interpolation regimes, a more refined approximation is that the SGD noise covariance aligns with both the loss value and the Hessian~\citep{mori2022power,wu2022alignment,wang2023noise}: $\bbE[  \bxi_t\bxi_t^\top]\simeq 2\cL(\btheta_t)  \nabla^2\cL(\btheta_t)$. Then
    \begin{align*}
        \cP(\btheta_t) \simeq2\cL(\btheta_t){\color{red!70!black}\|\nabla^2\cL(\btheta_t)\|_{\rm F}^2}.
    \end{align*}
        
    \end{itemize}
    
    Overall, across these noise assumptions, the SGD noise contribution to the loss is governed
    by \textbf{Hessian sharpness metrics}, such as
    \[
        \lambda_{\max}(\nabla^2\cL(\btheta_t)),
        \quad
        \Tr(\nabla^2\cL(\btheta_t)),
        \quad
        \|\nabla^2\cL(\btheta_t)\|_{\rm F}^2.
    \]
\end{itemize}

\section{Proofs in Section~\ref{section: improving}}
\label{appendix: proof of improving}

\subsection{Proofs in Section~\ref{subsection: placement}}
\label{appendix: proof of improving: placement}

\textbf{Proof Notations.}
We compare the input-only scale-vector model
\[
    f(\bx;\bU,\bgamma)
    =
    \bU(\bgamma\odot \Norm(\bx))
\]
with the input-output scale-vector model
\[
    \phi(\bx;\bM,\bgamma^a,\bgamma^b)
    =
    \bgamma^a\odot\left(\bM(\bgamma^b\odot \Norm(\bx))\right).
\]
Here \(\bgamma\in\mathbb R^d\) is the input scale vector of the model \(f\),
whereas \(\bgamma^a\in\mathbb R^c\) and \(\bgamma^b\in\mathbb R^d\) are the
output and input scale vectors of the model \(\phi\), respectively. Define
\[
    \bD_\gamma=\operatorname{diag}(\bgamma),
    \qquad
    \bD_a=\operatorname{diag}(\bgamma^a),
    \qquad
    \bD_b=\operatorname{diag}(\bgamma^b).
\]
The effective weights are
\[
    \bA_f=\bU\bD_\gamma,
    \qquad
    \bA_\phi=\bD_a\bM\bD_b.
\]
The corresponding losses are
\[
    \mathcal L_f(\bU,\bgamma)
    =
    \frac12\|\bU\bD_\gamma-\bW^\star\|_F^2,
\]
and
\[
    \mathcal L_\phi(\bM,\bgamma^a,\bgamma^b)
    =
    \frac12\|\bD_a\bM\bD_b-\bW^\star\|_F^2.
\]
For a generic effective state \(\bA\), define the residual
\[
    \bR=\bW^\star-\bA.
\]

\begin{theorem}[Optimization advantage of input-output scale vectors]\label{thm: full: optimization advantage placement}
Consider the setting above. Under gradient flow, the following statements hold.
\begin{itemize}
    \item \textbf{(i) Instantaneous acceleration at the same effective state.}
    Suppose that $f$ and $\phi$ are evaluated at parameter values satisfying $\bU\bD_\gamma=\bD_a\bM\bD_b,\bgamma=\bgamma^b$.
    Then their instantaneous loss derivatives satisfy
    \[
        \frac{\rd}{\rd t}\mathcal L_\phi
        \le
        \frac{\rd}{\rd t}\mathcal L_f
        \le
        0.
    \]

    \item \textbf{(ii) Local strict acceleration from zero initialization.}
    Let $\bU(0)=\bzero,\quad
        \bM(0)=\bzero,\quad
        \bgamma(0)=\bone,\quad
        \bgamma^a(0)=\bone,\quad
        \bgamma^b(0)=\bone$.
    If $\bW^\star\neq\bzero$, then, for all sufficiently small $t>0$,
    \[
        \mathcal L_\phi(t)<\mathcal L_f(t).
    \]
    More precisely,
    \[
        \mathcal L_\phi(t)-\mathcal L_f(t)
        =
        -\frac{2}{3}t^3
        \sum_{i=1}^c
        \|\bW^\star_{i:}\|_2^4
        +
        \mathcal O(t^4).
    \]

    \item \textbf{(iii) Global strict acceleration under a matching-support teacher.}
    Assume $\bW^\star\neq\bzero$ and that its support forms a matching, namely
    \[
        \|\bW^\star_{i:}\|_0\le 1
        \quad\text{for every } i,
        \qquad
        \|\bW^\star_{:j}\|_0\le 1
        \quad\text{for every } j.
    \]
    Under the standard initialization $\bU(0)=\bzero,\quad
        \bM(0)=\bzero,\quad
        \bgamma(0)=\bone,\quad
        \bgamma^a(0)=\bone,\quad
        \bgamma^b(0)=\bone$,
    we have
    \[
        \mathcal L_\phi(t)<\mathcal L_f(t),
        \qquad
        \forall t>0.
    \]
\end{itemize}
\end{theorem}

\begin{proof}[Proof of Claim (i) in Theorem~\ref{thm: full: optimization advantage placement}]
For \(f\), since
\[
    \mathcal L_f=\frac12\|\bU\bD_\gamma-\bW^\star\|_F^2
    =
    \frac12\|\bR\|_F^2,
\]
the gradient-flow equations are
\[
    \dot{\bU}=\bR\bD_\gamma,
    \qquad
    \dot\gamma_j
    =
    \langle \bU_{:j},\bR_{:j}\rangle.
\]
Therefore,
\[
    -\frac{d \cL_f}{dt}
    =
    \|\bR\bD_\gamma\|_F^2
    +
    \sum_{j=1}^d
    \langle \bU_{:j},\bR_{:j}\rangle^2.
\]
Using \(\bU_{:j}=\bA_{:j}/\gamma_j\), we obtain
\[
    -\frac{d \cL_f}{dt}
    =
    \sum_{j=1}^d
    \gamma_j^2\|\bR_{:j}\|_2^2
    +
    \sum_{j=1}^d
    \frac{\langle \bA_{:j},\bR_{:j}\rangle^2}
    {\gamma_j^2}.
\]

For \(\phi\), the gradient-flow equations are
\[
    \dot{\bM}=\bD_a\bR\bD_b,
\]
\[
    \dot\gamma_i^a
    =
    \langle \bR_{i:},(\bM\bD_b)_{i:}\rangle,
\]
and
\[
    \dot\gamma_j^b
    =
    \langle \bR_{:j},(\bD_a\bM)_{:j}\rangle.
\]
Hence
\[
    -\frac{d \cL_\phi}{dt}
    =
    \|\bD_a\bR\bD_b\|_F^2
    +
    \sum_{i=1}^c
    \langle \bR_{i:},(\bM\bD_b)_{i:}\rangle^2
    +
    \sum_{j=1}^d
    \langle \bR_{:j},(\bD_a\bM)_{:j}\rangle^2.
\]
Because
\[
    (\bM\bD_b)_{i:}=\frac{\bA_{i:}}{\gamma_i^a},
    \qquad
    (\bD_a\bM)_{:j}=\frac{\bA_{:j}}{\gamma_j^b},
\]
we get
\[
    -\frac{d \cL_\phi}{dt}
    =
    \sum_{i=1}^c\sum_{j=1}^d
    (\gamma_i^a)^2(\gamma_j^b)^2R_{ij}^2
    +
    \sum_{i=1}^c
    \frac{\langle \bA_{i:},\bR_{i:}\rangle^2}
    {(\gamma_i^a)^2}
    +
    \sum_{j=1}^d
    \frac{\langle \bA_{:j},\bR_{:j}\rangle^2}
    {(\gamma_j^b)^2}.
\]
Using the matched-input-scale assumption \(\bgamma=\bgamma^b\), subtracting the
expression for \(\mathcal D_f\) proves 
\[
    \frac{d \cL_f}{dt} -\frac{d \cL_\phi}{dt} 
    =
    \sum_{i=1}^c\sum_{j=1}^d
    \left((\gamma_i^a)^2-1\right)
    \gamma_j^2 R_{ij}^2
    +
    \sum_{i=1}^c
    \frac{
        \langle \bA_{i:},\bR_{i:}\rangle^2
    }{
        (\gamma_i^a)^2
    }.
\]
Consequently, due to the conservation law of gradient flow
\[
\frac{\rd}{\rd t}\left((\gamma_i^a)^2-\|\bM_{i:}\|_2^2\right)=0,
\]
we have
\[
    (\gamma_i^a(t))^2=\|\bM_{i:}(t)\|_2^2+(\gamma_i^a(0))^2-\|\bM_{i:}(0)\|_2^2\ge 1,
    \qquad i=1,\dots,c.
\]
Therefore,
\[
    \frac{d \cL_f}{dt} -\frac{d \cL_\phi}{dt}\geq0.
\]

\end{proof}

\begin{proof}[Proof of Claim (ii) in Theorem~\ref{thm: full: optimization advantage placement}]
At \(t=0\), both effective weights are zero:
\[
    \bA_f(0)=\bA_\phi(0)=\bzero.
\]
Moreover, since all scale vectors are initialized at one,
\[
    \dot{\bU}(0)=\bW^\star,
    \qquad
    \dot{\bM}(0)=\bW^\star.
\]
Thus
\[
    \bU(t)=t\bW^\star+O(t^2),
    \qquad
    \bM(t)=t\bW^\star+O(t^2).
\]
For the input-only model \(f\), its input scale obeys
\[
    \dot\gamma_j
    =
    \langle \bU_{:j},\bR_{:j}\rangle.
\]
Hence
\[
    \gamma_j(t)
    =
    1+\frac12\|\bW^\star_{:j}\|_2^2t^2+O(t^3).
\]
For the input-output model \(\phi\), the output scale obeys
\[
    \dot\gamma_i^a
    =
    \langle \bR_{i:},(\bM\bD_b)_{i:}\rangle.
\]
Using \(\bM(t)=t\bW^\star+O(t^2)\) and \(\bR(t)=\bW^\star+O(t)\), we obtain
\[
    \gamma_i^a(t)
    =
    1+\frac12\|\bW^\star_{i:}\|_2^2t^2+O(t^3).
\]
The input scale in \(\phi\) satisfies the same second-order expansion as the
input scale in \(f\):
\[
    \gamma_j^b(t)
    =
    1+\frac12\|\bW^\star_{:j}\|_2^2t^2+O(t^3).
\]

We now compare the instantaneous loss-decrease rates. Since
\[
    \bA_f(t)=t\bW^\star+O(t^2),
    \qquad
    \bA_\phi(t)=t\bW^\star+O(t^2),
\]
and both residuals are
\[
    \bR_f(t)=\bW^\star+O(t),
    \qquad
    \bR_\phi(t)=\bW^\star+O(t),
\]
a direct Taylor expansion of the gradient-flow loss-decrease rates gives
\[
    \mathcal D_\phi(t)-\mathcal D_f(t)
    =
    2t^2
    \sum_{i=1}^c
    \|\bW^\star_{i:}\|_2^4
    +
    O(t^3).
\]
Indeed, the leading difference is exactly the row-wise output-scale
contribution in \(\phi\), while the input-scale contributions agree up to this
order. Since
\[
    \frac{\rd}{\rd t}
    \left(\mathcal L_\phi(t)-\mathcal L_f(t)\right)
    =
    -\left(\mathcal D_\phi(t)-\mathcal D_f(t)\right),
\]
and \(\mathcal L_\phi(0)=\mathcal L_f(0)\), integration from \(0\) to \(t\)
gives
\[
    \mathcal L_\phi(t)-\mathcal L_f(t)
    =
    -\frac23t^3
    \sum_{i=1}^c
    \|\bW^\star_{i:}\|_2^4
    +
    O(t^4).
\]
If \(\bW^\star\neq\bzero\), then at least one row has positive norm, so
\[
    \sum_{i=1}^c\|\bW^\star_{i:}\|_2^4>0.
\]
Therefore \(\mathcal L_\phi(t)<\mathcal L_f(t)\) for all sufficiently small
\(t>0\).
\end{proof}

\begin{proof}[Proof of Claim (iii) in Theorem~\ref{thm: full: optimization advantage placement}]
Because \(\bW^\star\) has matching support and the zero entries are initialized
at zero, the gradient-flow dynamics preserves the support. Indeed, if
\(W^\star_{ij}=0\) and \(M_{ij}(0)=U_{ij}(0)=0\), then initially the effective
weight at \((i,j)\) is zero, hence the residual at \((i,j)\) is zero. The
gradient equations give
\[
    \dot U_{ij}=\gamma_j R_{ij},
    \qquad
    \dot M_{ij}=\gamma_i^a R_{ij}\gamma_j^b,
\]
so \(U_{ij}\) and \(M_{ij}\) remain zero. Thus the system decomposes into
independent scalar problems, one for each nonzero entry of \(\bW^\star\).

It remains to prove the scalar claim. Let \(w^\star\neq 0\). By symmetry, we
assume \(w^\star>0\). The input-only model has effective weight
\[
    a_f=u\gamma.
\]
Its gradient flow satisfies
\[
    \dot u=\gamma(w^\star-a_f),
    \qquad
    \dot\gamma=u(w^\star-a_f).
\]
The conservation law
\[
    \gamma^2-u^2=1
\]
holds. Therefore
\[
    \gamma^2+u^2=\sqrt{1+4a_f^2},
\]
and hence
\[
    \dot a_f
    =
    \sqrt{1+4a_f^2}\,(w^\star-a_f).
\]

For the input-output model, the effective weight is
\[
    a_\phi=\gamma^a m\gamma^b.
\]
The scalar gradient flow satisfies
\[
    \dot m=\gamma^a\gamma^b(w^\star-a_\phi),
\]
\[
    \dot\gamma^a=m\gamma^b(w^\star-a_\phi),
    \qquad
    \dot\gamma^b=\gamma^a m(w^\star-a_\phi).
\]
The conservation laws are
\[
    (\gamma^a)^2-m^2=1,
    \qquad
    (\gamma^b)^2-m^2=1.
\]
Since \(\gamma^a(0)=\gamma^b(0)=1\), we have
\[
    \gamma^a(t)=\gamma^b(t)=\sqrt{1+m(t)^2}.
\]
Thus
\[
    a_\phi=m(1+m^2).
\]
The effective dynamics is
\[
    \dot a_\phi
    =
    (1+4m^2+3m^4)(w^\star-a_\phi).
\]

We compare the two scalar speed factors at the same effective weight
\(a>0\). For \(f\),
\[
    q_f(a)=\sqrt{1+4a^2}.
\]
For \(\phi\), let \(m=m(a)>0\) be the unique solution of
\[
    a=m(1+m^2).
\]
Then
\[
    q_\phi(a)=1+4m^2+3m^4.
\]
Let \(y=m^2>0\). Since \(a^2=y(1+y)^2\), we have
\[
    q_f(a)^2=1+4y(1+y)^2,
\]
while
\[
    q_\phi(a)^2=(1+4y+3y^2)^2.
\]
A direct calculation gives
\[
    q_\phi(a)^2-q_f(a)^2
    =
    y(9y^3+20y^2+14y+4)>0.
\]
Hence
\[
    q_\phi(a)>q_f(a),
    \qquad
    \forall a>0.
\]

Both effective weights solve scalar ODEs of the form
\[
    \dot a=q(a)(w^\star-a),
    \qquad
    a(0)=0,
\]
where \(q_\phi(a)>q_f(a)\) for all \(a>0\). The scalar comparison principle
implies
\[
    a_\phi(t)>a_f(t),
    \qquad
    \forall t>0.
\]
Both trajectories remain below \(w^\star\), so
\[
    |w^\star-a_\phi(t)|<|w^\star-a_f(t)|.
\]
Therefore
\[
    \frac12(w^\star-a_\phi(t))^2
    <
    \frac12(w^\star-a_f(t))^2,
    \qquad
    \forall t>0.
\]

Applying this scalar comparison independently to every nonzero entry of
\(\bW^\star\), and summing the scalar losses, yields
\[
    \mathcal L_\phi(t)<\mathcal L_f(t),
    \qquad
    \forall t>0.
\]
\end{proof}

\subsection{Proofs in Section~\ref{subsection: reparameterization}}
\label{appendix: proof of improving: reparameterization}

\textbf{Proof Notations.}
We compare the standard scale-vector model
\[
    f(\bx;\bU,\bgamma)
    =
    \bU(\bgamma\odot \Norm(\bx))
\]
with the OR-reparameterized scale-vector model
\[
    \psi(\bx;\bV,\balpha,\beta)=\bV(\beta\Norm(\balpha)\odot\Norm(\bx)).
\]

The corresponding effective weights are
\[
    \bA_f=\bU\operatorname{diag}(\bgamma),
    \qquad
    \bA_\psi=\bV\operatorname{diag}(\beta\Norm(\balpha)).
\]
We denote their losses by $\cL_f$ and $\cL_\psi$, respectively.

\begin{theorem}[Optimization advantage of \textbf{OR} reparameterization]
\label{thm: full: optimization advantage reparametrization}
Consider the setting in Theorem~\ref{thm: optimization advantage w/ scale vector}. We compare the standard model $f$ with the reparameterized model $\psi$ under
gradient flow. The parameters are initialized as $\balpha(0)=\bone,\beta(0)=1,\bV(0)=\bzero$; $\bgamma(0)=\bone,\bU(0)=\bzero$. 
Then the following statements hold.
\begin{itemize}
    \item \textbf{(i) Instantaneous acceleration at the same effective state.}
    For any $(\bU,\bgamma)$ in $f$ and $(\bV,\balpha,\beta)$ in $\psi$ satisfying $\bU=\bV$ and $\bgamma=\beta\Norm(\balpha)$,  the \textbf{OR} reparameterization is at least as fast as the direct scale-vector parameterization:
    \[\frac{\rd}{\rd t}\cL_{\psi}\leq\frac{\rd}{\rd t}\cL_f\leq 0.\]
    \item \textbf{(ii) Local strict acceleration from initialization.}
    If $d>1$, then the \textbf{OR} reparameterization is strictly faster in the early
    phase:
    \[\cL_{\psi}(t)-\cL_f(t)=-\frac{2}{3}(1-1/d)\|\bW_\star\|_{\rm F}^4 t^3+\cO(t^4).\]
    \item \textbf{(iii) Global strict acceleration under balanced teacher columns.}
    Suppose $d>1$ and $\|\bw_{\star,1}\|_2=\cdots=\|\bw_{\star,d}\|_2>0$, where $\bw_{\star,j}$ denotes the $j$-th column of $\bW_\star$. Then the
    \textbf{OR} reparameterization is strictly faster throughout training
    \[
    \cL_\psi(t)<\cL_f(t), \forall t>0.
    \]
\end{itemize}
\end{theorem}

In addition, in Figure~\ref{fig: numerical validation, OR accelerate}, we provide numerical validation of three claims in Theorem~\ref{thm: full: optimization advantage reparametrization}. The simulations support the theoretically predicted acceleration induced by \textbf{OR}.

\begin{figure}[!htp]
    \centering
    \includegraphics[width=0.32\linewidth]{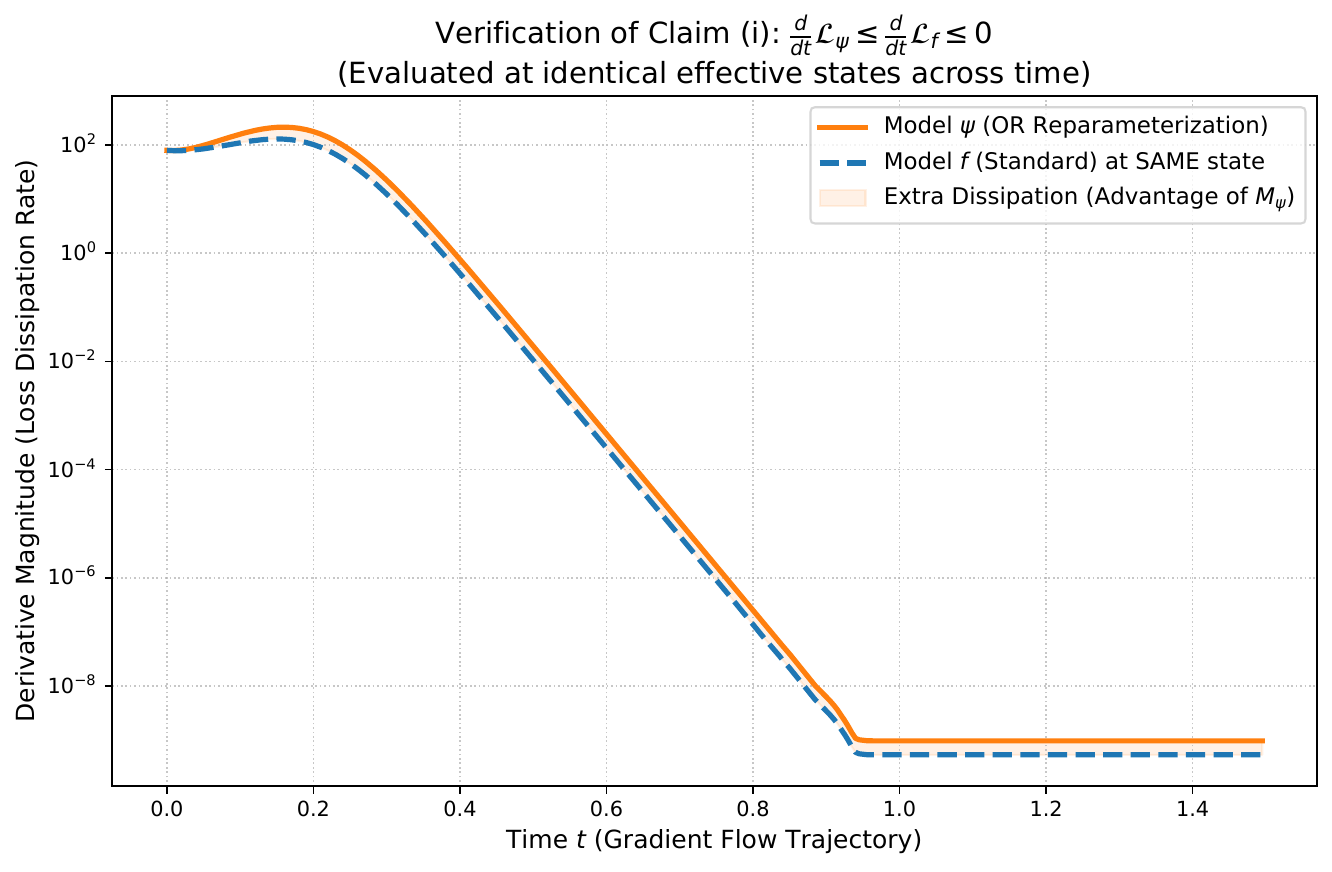}
    \includegraphics[width=0.32\linewidth]{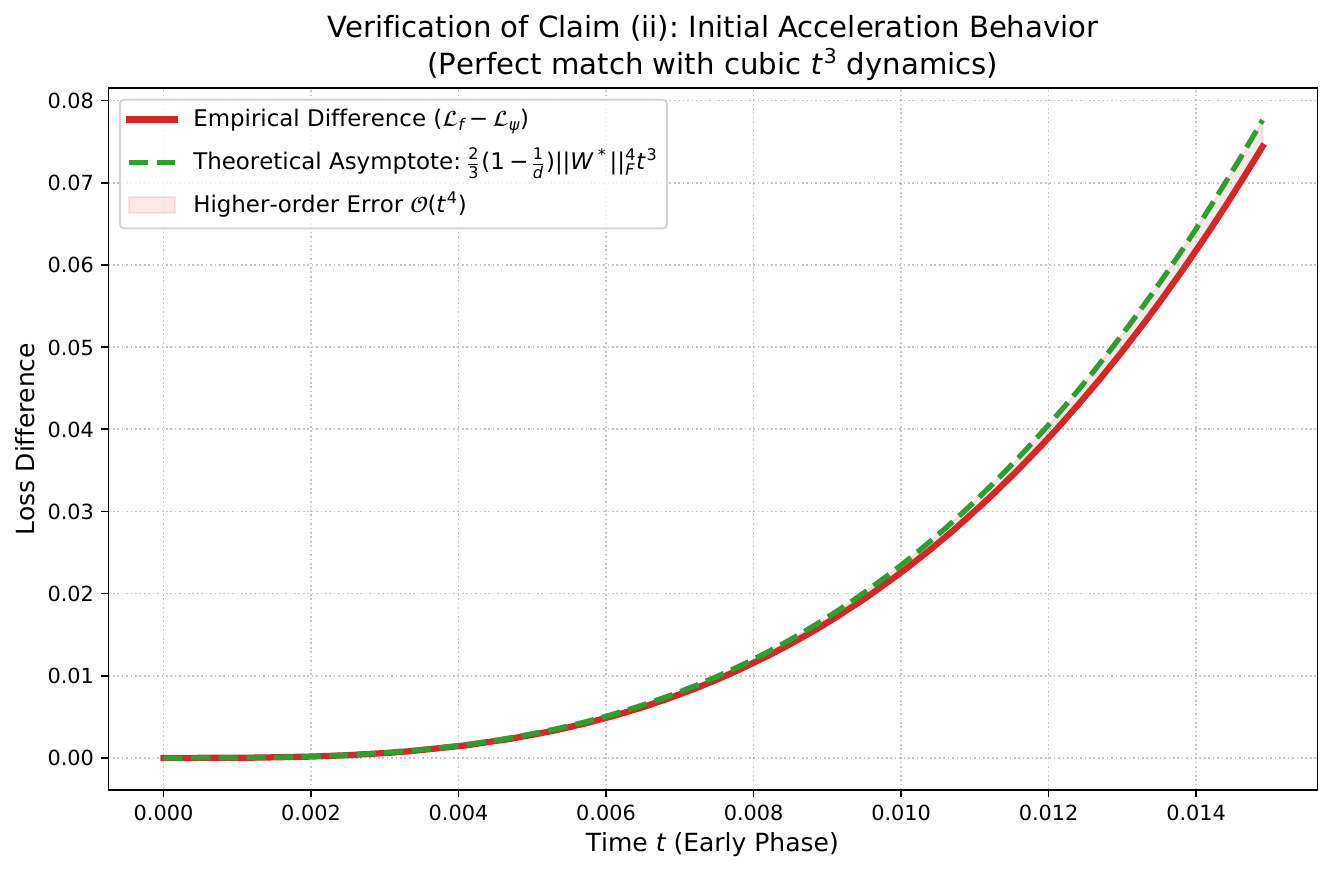}
    \includegraphics[width=0.33\linewidth]{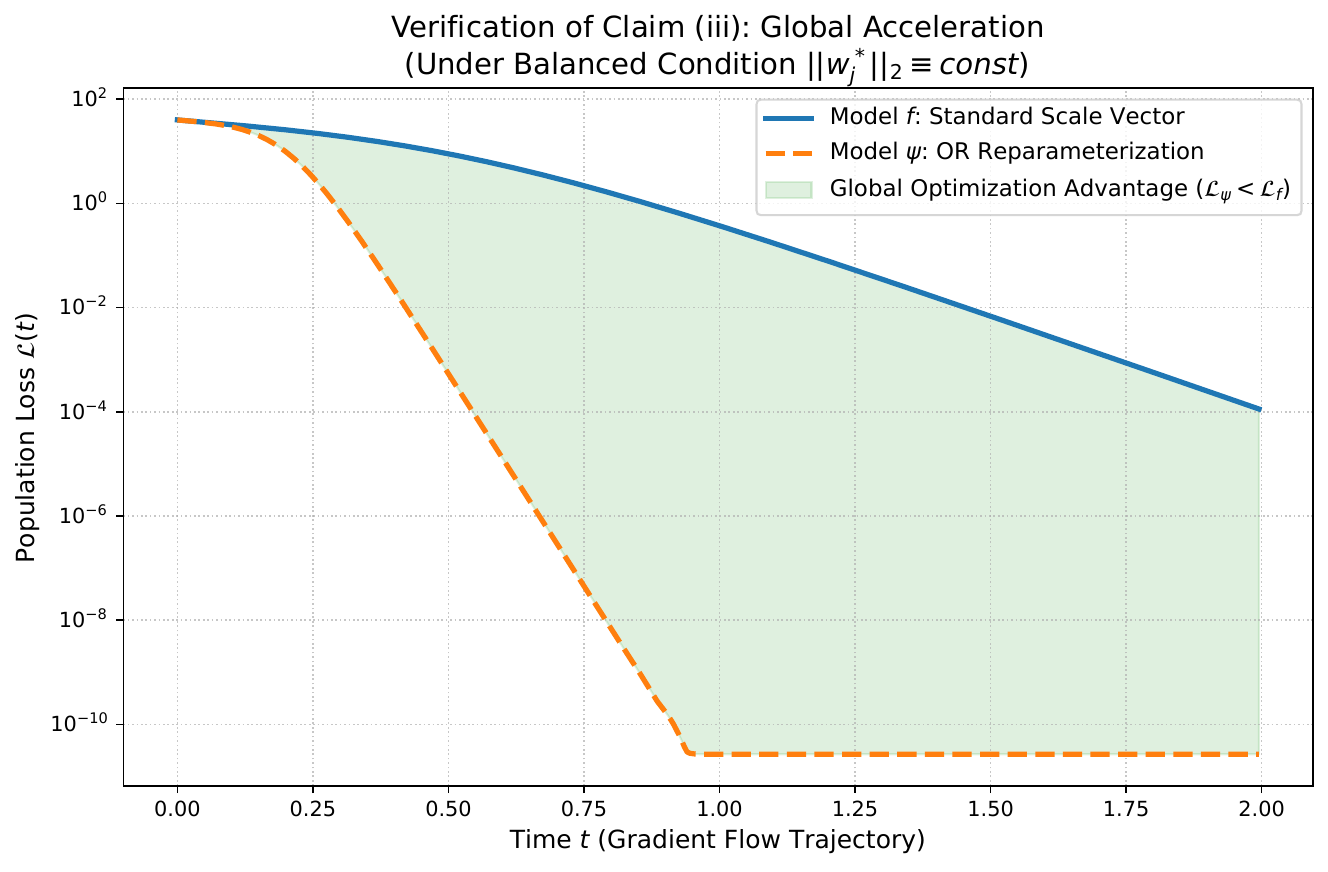}
    \caption{Numerical validation of the claims in Theorem~\ref{thm: full: optimization advantage reparametrization}.}
    \label{fig: numerical validation, OR accelerate}
\end{figure}

\begin{proof}[Proof of Claim (i) in Theorem~\ref{thm: full: optimization advantage reparametrization}]
Consider the model
\[
    \psi(\bx;\bV,\beta,\balpha)
    =
    \bV\bigl(\beta \Norm(\balpha)\odot \Norm(\bx)\bigr),
\]
where
\[
    \bV\in\mathbb R^{c\times d},
    \qquad
    \balpha\in\mathbb R^d,
    \qquad
    \beta\in\mathbb R.
\]
Let
\[
    \bq=\Norm(\balpha)
    =
    \sqrt d\,\frac{\balpha}{\|\balpha\|_2},
    \qquad
    \bgamma=\beta\bq.
\]
Then
\[
    \psi(\bx;\bV,\beta,\balpha)
    =
    \bV(\bgamma\odot \Norm(\bx)).
\]
Thus $\psi$ can be viewed as a magnitude-direction reparameterization of the
scale vector $\bgamma$ in $f$.

As before, since
\[
    \mathbb E[\Norm(\bx)\Norm(\bx)^\top]=\bI_d,
\]
the loss of $\psi$ is
\[
    \mathcal L_\psi(\bV,\beta,\balpha)
    =
    \frac12\|\bV\bD_{\bgamma}-\bW^\star\|_F^2,
    \qquad
    \bD_{\bgamma}=\operatorname{diag}(\bgamma).
\]
Write
\[
    \bV=[\bv_1,\dots,\bv_d],
    \qquad
    \bW^\star=[\bw_1^\star,\dots,\bw_d^\star],
\]
and define
\[
    \br_j=\bw_j^\star-\gamma_j\bv_j,
    \qquad
    s_j=\bv_j^\top\br_j,
    \qquad
    \bs=(s_1,\dots,s_d)^\top.
\]
Then
\[
    \mathcal L_\psi
    =
    \frac12\sum_{j=1}^d\|\br_j\|_2^2.
\]

The gradient-flow equations are
\[
    \dot{\bv}_j=\gamma_j\br_j,
\]
\[
    \dot\beta=\bq^\top\bs,
\]
and
\[
    \dot{\balpha}
    =
    \beta \bJ_{\bq}(\balpha)^\top \bs,
\]
where
\[
    \bJ_{\bq}(\balpha)
    =
    \frac{\partial \bq}{\partial \balpha}
    =
    \frac{\sqrt d}{\|\balpha\|_2}
    \left(
        \bI_d-\frac{\balpha\balpha^\top}{\|\balpha\|_2^2}
    \right).
\]
Since $\balpha(0)=\bone$, we have $\|\balpha(0)\|_2=\sqrt d$. Moreover,
$\dot{\balpha}$ is always orthogonal to $\balpha$, so
\[
    \|\balpha(t)\|_2=\sqrt d,
    \qquad \forall t\ge 0.
\]
Therefore,
\[
    \bJ_{\bq}(\balpha(t))
    =
    \bP_{\bq(t)}
    :=
    \bI_d-\frac{\bq(t)\bq(t)^\top}{d}.
\]
Consequently,
\[
    \dot{\bq}
    =
    \beta \bP_{\bq}\bs.
\]
Since $\bgamma=\beta\bq$, we obtain
\[
    \dot{\bgamma}
    =
    \dot\beta\,\bq+\beta\dot{\bq}
    =
    \bq\bq^\top\bs+\beta^2\bP_{\bq}\bs.
\]
Equivalently,
\[
    \dot{\bgamma}
    =
    \bM_\psi\bs,
    \qquad
    \bM_\psi
    :=
    \bq\bq^\top+\beta^2\bP_{\bq}.
\]

For comparison, in the original scale-vector model $f$, the corresponding
gradient-flow equation is simply
\[
    \dot{\bgamma}=\bs.
\]
Thus, $\psi$ replaces the Euclidean scale-vector dynamics by a preconditioned
dynamics with preconditioner
\[
    \bM_\psi
    =
    \bq\bq^\top+\beta^2\bP_{\bq}.
\]

Moreover, $\psi$ satisfies the conservation law
\[
    \frac{\rd}{\rd t}
    \left(
        \beta^2-\|\bV\|_F^2
    \right)
    =
    0.
\]
Indeed,
\[
    \frac{\rd}{\rd t}\|\bV\|_F^2
    =
    2\sum_{j=1}^d \bv_j^\top\dot{\bv}_j
    =
    2\sum_{j=1}^d \gamma_j \bv_j^\top\br_j
    =
    2\beta \bq^\top\bs,
\]
while
\[
    \frac{\rd}{\rd t}\beta^2
    =
    2\beta\dot\beta
    =
    2\beta\bq^\top\bs.
\]
Since $\beta(0)=1$ and $\bV(0)=\bzero$, we have
\[
    \beta(t)^2-\|\bV(t)\|_F^2=1.
\]
In particular,
\[
    \beta(t)^2\ge 1.
\]

Now the loss dissipation of $\psi$ is
\[
    \frac{\rd}{\rd t}\mathcal L_\psi
    =
    -\sum_{j=1}^d \gamma_j^2\|\br_j\|_2^2
    -
    (\bq^\top\bs)^2
    -
    \beta^2\|\bP_{\bq}\bs\|_2^2.
\]
By contrast, for $f$ at the same effective state, the loss dissipation is
\[
    \frac{\rd}{\rd t}\mathcal L_f
    =
    -\sum_{j=1}^d \gamma_j^2\|\br_j\|_2^2
    -
    \|\bs\|_2^2.
\]
Decompose
\[
    \bs=\bs_\parallel+\bs_\perp,
    \qquad
    \bs_\parallel
    =
    \frac{\bq^\top\bs}{d}\bq,
    \qquad
    \bs_\perp=\bP_{\bq}\bs.
\]
Since $\|\bq\|_2^2=d$, we have
\[
    (\bq^\top\bs)^2=d\|\bs_\parallel\|_2^2.
\]
Therefore,
\[
    (\bq^\top\bs)^2+\beta^2\|\bP_{\bq}\bs\|_2^2
    =
    d\|\bs_\parallel\|_2^2
    +
    \beta^2\|\bs_\perp\|_2^2.
\]
Since $d\ge 1$ and $\beta^2\ge 1$, this implies
\[
    (\bq^\top\bs)^2+\beta^2\|\bP_{\bq}\bs\|_2^2
    \ge
    \|\bs_\parallel\|_2^2+\|\bs_\perp\|_2^2
    =
    \|\bs\|_2^2.
\]
Hence, at the same effective state, the magnitude-direction parameterization
$\psi$ dissipates the loss at least as fast as the direct scale-vector
parameterization $f$.
\end{proof}

\begin{proof}[Proof of Claim (ii) in Theorem~\ref{thm: full: optimization advantage reparametrization}]
For both models, at initialization,
\[
    \bv_j(0)=\bu_j(0)=\bzero,
    \qquad
    \br_j(0)=\bw_j^\star.
\]
Thus
\[
    \bv_j(t)=t\bw_j^\star+O(t^2),
    \qquad
    \bu_j(t)=t\bw_j^\star+O(t^2),
\]
and hence
\[
    s_j(t)
    =
    \bv_j(t)^\top\br_j(t)
    =
    t\|\bw_j^\star\|_2^2+O(t^2)
    =
    t\rho_j+O(t^2).
\]

For $f$, the scale-vector dynamics are
\[
    \dot\gamma_j=s_j.
\]
Therefore,
\[
    \gamma_j^{(f)}(t)
    =
    1+\frac12\rho_j t^2+O(t^3).
\]

For $\psi$, at initialization we have
\[
    \bq(0)=\bone,
    \qquad
    \beta(0)=1,
    \qquad
    \bP_{\bq(0)}
    =
    \bI_d-\frac1d\bone\bone^\top.
\]
The induced scale-vector dynamics are
\[
    \dot{\bgamma}
    =
    \bq\bq^\top\bs+\beta^2\bP_{\bq}\bs.
\]
Thus, at leading order,
\[
    \dot\gamma_j^{(\psi)}(t)
    =
    t\rho_j
    +
    \left(1-\frac1d\right)tS
    +
    O(t^2).
\]
Hence
\[
    \gamma_j^{(\psi)}(t)
    =
    1+
    \frac12
    \left[
        \rho_j+
        \left(1-\frac1d\right)S
    \right]t^2
    +
    O(t^3).
\]

Let
\[
    D_f(t):=-\frac{\rd}{\rd t}\mathcal L_f(t),
    \qquad
    D_\psi(t):=-\frac{\rd}{\rd t}\mathcal L_\psi(t).
\]
Using the loss-dissipation identities,
\[
    D_f(t)
    =
    \sum_{j=1}^d
    \left(\gamma_j^{(f)}(t)\right)^2
    \|\br_j^{(f)}(t)\|_2^2
    +
    \|\bs^{(f)}(t)\|_2^2,
\]
while
\[
    D_\psi(t)
    =
    \sum_{j=1}^d
    \left(\gamma_j^{(\psi)}(t)\right)^2
    \|\br_j^{(\psi)}(t)\|_2^2
    +
    (\bq(t)^\top\bs^{(\psi)}(t))^2
    +
    \beta(t)^2\|\bP_{\bq(t)}\bs^{(\psi)}(t)\|_2^2.
\]
Substituting the expansions above gives
\[
    D_\psi(t)-D_f(t)
    =
    2
    \left(1-\frac1d\right)
    S^2 t^2
    +
    O(t^3).
\]
Therefore,
\[
    \mathcal L_\psi(t)-\mathcal L_f(t)
    =
    -\int_0^t
    \left(D_\psi(s)-D_f(s)\right)\,ds
\]
\[
    =
    -\frac{2}{3}
    \left(1-\frac1d\right)
    S^2 t^3
    +
    O(t^4).
\]
This proves the claim.
\end{proof}

\begin{proof}[Proof of Claim (iii) in Theorem~\ref{thm: full: optimization advantage reparametrization}]
Under the balanced condition and the symmetric initialization
\[
    \bq(0)=\bone,
\]
the symmetry is preserved along the gradient flow of $\psi$. Indeed, the vector
$\bs$ remains proportional to $\bone$, and hence
\[
    \bP_{\bq}\bs=\bzero.
\]
Therefore,
\[
    \bq(t)=\bone,
    \qquad
    \bgamma(t)=\beta(t)\bone.
\]

For both models, each column remains aligned with its teacher column. Thus we
may write
\[
    \bu_j(t)=p_f(t)\bw_j^\star
\]
for $f$, and
\[
    \bv_j(t)=p_\psi(t)\bw_j^\star
\]
for $\psi$.

For $f$, the scalar dynamics are
\[
    \dot p_f=\gamma_f(1-\gamma_f p_f),
    \qquad
    \dot\gamma_f=\rho p_f(1-\gamma_f p_f),
\]
with
\[
    p_f(0)=0,
    \qquad
    \gamma_f(0)=1.
\]
For $\psi$, the scalar dynamics are
\[
    \dot p_\psi=\beta(1-\beta p_\psi),
    \qquad
    \dot\beta=d\rho p_\psi(1-\beta p_\psi),
\]
with
\[
    p_\psi(0)=0,
    \qquad
    \beta(0)=1.
\]

Both dynamics have the same form
\[
    \dot p=\eta(1-\eta p),
    \qquad
    \dot\eta=k\rho p(1-\eta p),
\]
where $k=1$ for $f$ and $k=d$ for $\psi$.

The corresponding invariant is
\[
    \eta^2-k\rho p^2=1.
\]
Let
\[
    \theta=\eta p.
\]
Then
\[
    \dot\theta
    =
    (\eta^2+k\rho p^2)(1-\theta).
\]
Using the invariant, set
\[
    z=k\rho p^2.
\]
Then
\[
    \eta^2=1+z,
    \qquad
    \theta^2
    =
    \eta^2p^2
    =
    \frac{z(1+z)}{k\rho}.
\]
Therefore,
\[
    z
    =
    \frac{\sqrt{1+4k\rho\theta^2}-1}{2}.
\]
Hence the effective coefficient $\theta$ satisfies the closed scalar ODE
\[
    \dot\theta
    =
    \left[
        1
        +
        2\cdot
        \frac{\sqrt{1+4k\rho\theta^2}-1}{2}
    \right]
    (1-\theta).
\]
Equivalently,
\[
    \dot\theta
    =
    \sqrt{1+4k\rho\theta^2}\,(1-\theta).
\]
For $\psi$, $k=d$, while for $f$, $k=1$. Since $d>1$, we have
\[
    \sqrt{1+4d\rho\theta^2}
    >
    \sqrt{1+4\rho\theta^2}
\]
for every $\theta\in(0,1)$. By the scalar comparison principle,
\[
    \theta_\psi(t)>\theta_f(t),
    \qquad
    \forall t>0.
\]
The residual coefficient is $1-\theta$, so
\[
    1-\theta_\psi(t)<1-\theta_f(t).
\]
Therefore,
\[
    \mathcal L_\psi(t)
    =
    \frac{d\rho}{2}(1-\theta_\psi(t))^2
    <
    \frac{d\rho}{2}(1-\theta_f(t))^2
    =
    \mathcal L_f(t),
    \qquad
    \forall t>0.
\]
This proves the claim.
\end{proof}

\end{document}